%% file: fullversion.tex
\RequirePackage{etex}
\documentclass[12pt]{article}

\usepackage[margin=.95in]{geometry} 
\usepackage[utf8]{inputenc} 
\usepackage[T1]{fontenc}    
\usepackage[colorlinks]{hyperref}
\usepackage{lmodern}
\usepackage{url}            
\usepackage{booktabs}       
\usepackage{amsfonts}       
\usepackage{nicefrac}       
\usepackage{microtype}      
\usepackage{xcolor}         
\usepackage{comment}
\usepackage{amsmath,amsthm,amssymb,multirow,csquotes,cleveref}
\usepackage{xhfill,tikz-network}
\usepackage{wrapfig}
\usepackage[export]{adjustbox}
\usepackage{graphicx,environ}
\usepackage{tikz}
\usepackage{tikzscale}
\usepackage{algorithm}
\usepackage{algorithmic}
\usepackage{pgfplots,enumitem,color,colortbl}
\usepackage{setspace}
\pgfplotsset{compat=newest}
\usetikzlibrary{automata,arrows,calc,positioning}

\DeclareMathOperator{\ex}{\mathbb{E}}

\DeclareMathOperator*{\argmin}{argmin}

\newtheorem{theorem}{Theorem}[section]

\newtheorem{claim}[theorem]{Claim}

\newtheorem{lemma}[theorem]{Lemma}
\newtheorem{proposition}[theorem]{Proposition}

\theoremstyle{definition}
\newtheorem{definition}{Definition}
\newtheorem{example}{Example}
\theoremstyle{remark}
\newtheorem{remark}{Remark}
\crefname{property}{property}{properties}

\makeatletter
\newsavebox{\measure@tikzpicture}
\NewEnviron{scaletikzpicturetowidth}[1]{%
  \def\tikz@width{#1}%
  \def\tikzscale{1}\begin{lrbox}{\measure@tikzpicture}%
  \BODY
  \end{lrbox}%
  \pgfmathparse{#1/\wd\measure@tikzpicture}%
  \edef\tikzscale{\pgfmathresult}%
  \BODY
 }
\makeatother

\newcommand{\<}{\left(}
\renewcommand{\>}{\right)}
\newcommand{\norm}[1]{\left\| #1 \right\|}
\renewcommand{\Pr}{\mathrm{Pr}}
\newcommand{\xx}{\mathbf{x}}
\newcommand{\yy}{\mathbf{y}}
\newcommand{\lt}{\left[}
\newcommand{\rt}{\right]}
\newcommand{\A}{\mathcal{A}}
\renewcommand{\r}{\mathbf{r}}
\renewcommand{\S}{\mathcal{S}}
\newcommand{\X}{\mathcal{X}}
\newcommand{\N}{\mathcal{N}}
\renewcommand{\a}{\mathbf{a}}
\renewcommand{\P}{\mathcal{P}}
\newcommand{\G}{\mathcal{G}}

\tikzset{node distance=4.5cm, every state/.style={semithick,fill=gray!10}, every edge/.style={draw,->,>=stealth',auto,semithick}}

\title{Global Convergence of Multi-Agent Policy Gradient in Markov Potential Games}

\author{\and\and \;\;\;\; Stefanos Leonardos\\ \;\;\;\; SUTD 
\and \;\; Will Overman\\ \;\; UC Irvine \and\and\and\and Ioannis Panageas \\UC Irvine \and Georgios Piliouras\\ SUTD}

\date{}

\begin{document}

\maketitle

\input{abstract}
\input{introduction}
\input{model}
\input{properties}
\input{convergence}

\input{experiments}
\input{conclusions}

\section*{Acknowledgements}
This project is supported in part by NRF2019-NRF-ANR095 ALIAS grant, grant PIE-SGP-AI-2018-01, NRF 2018 Fellowship NRF-NRFF2018-07,  AME Programmatic Fund (Grant No. A20H6b0151) from the Agency for Science, Technology and Research (A*STAR) and the National Research Foundation, Singapore under its AI Singapore Program (AISG Award No: AISG2-RP-2020-016).

\bibliographystyle{plain}
\bibliography{bib_markov_potential}

\appendix
\input{appendix}

\end{document}

%% file: abstract.tex
\begin{abstract}
Potential games are arguably one of the most important and widely studied classes of normal form games. They define the archetypal setting of multi-agent coordination as all agent utilities are perfectly aligned with each other via a common potential function. Can this intuitive framework be transplanted in the setting of Markov Games? What are the similarities and differences between multi-agent coordination with and without state dependence? We present a novel definition of Markov Potential Games (MPG) that generalizes prior attempts at capturing complex stateful multi-agent coordination. Counter-intuitively, insights from normal-form potential games do not carry over as MPGs can consist of settings where state-games can be zero-sum games. In the opposite direction, Markov games where every state-game is a potential game are not necessarily MPGs. Nevertheless, MPGs showcase standard desirable properties such as the existence of deterministic Nash policies. In our main technical result, we prove (polynomially fast in the approximation error) convergence of independent policy gradient to Nash policies by adapting recent gradient dominance property arguments developed for single agent MDPs to multi-agent learning settings. 
\end{abstract}

%% file: introduction.tex
\section{Introduction}
Reinforcement learning (RL) has been a fundamental driver of numerous recent advances in Artificial Intelligence (AI) that range from super-human performance in competitive game-playing \cite{Sil16,Sil18,Bro19} and strategic decision-making in multiple tasks \cite{Mni15,Ope18,Vin19} to robotics, autonomous-driving and cyber-physical systems \cite{Bus08,Zha19}. A core ingredient for the success of single-agent RL systems, which are typically modelled as Markov Decision Processes (MDPs), is the existence of stationary deterministic optimal policies \cite{Ber00,Sut18}. This allows the design of efficient algorithms that provably converge towards such policies \cite{Aga20}. However, in practice, a majority of the above systems involve multi-agent interactions. In such cases, despite the notable empirical advancements, there is a lack of understanding of the theoretical convergence guarantees of existing multi-agent reinforcement learning (MARL) algorithms.\par
The main challenge when transitioning from single to multi-agent RL settings is the computation of \emph{Nash policies}. A Nash policy for $n>1$ agents is defined to be a profile of policies $(\pi_1^*,...,\pi_n^*)$ so that by fixing the stationary policies of all agents but $i$, $\pi_i^*$ is an optimal policy for the resulting single-agent MDP and this is true for all $1 \leq i \leq n$ \footnote{Analogue of Nash equilibrium notion.} (see Definition \ref{def:optimalpolicy}). Note that in multi-agent settings, Nash policies \textit{may not be unique} in principle.

A common approach for computing Nash policies in MDPs is the use of \emph{policy gradient} methods. There has been significant progress in the analysis of policy gradient methods during the last couple of years, notably including the works of \cite{Aga20} (and references therein), but it has mainly concerned the single-agent case: the convergence properties of policy gradient in MARL remain poorly understood. Existing steps towards a theory for multi-agent settings involve the papers of \cite{Das20} who show convergence of \textit{independent policy gradient} to the optimal policy for two-agent zero-sum stochastic games, of \cite{mdplastiterate} who improve the result of \cite{Das20} using optimistic policy gradient and of \cite{newZS} who study extensions of Natural Policy Gradient using function approximation. It is worth noting that the positive results of \cite{Das20, mdplastiterate} and \cite{newZS} depend on the fact that two-agent stochastic zero-sum games satisfy the \enquote{min-max equals max-min} property \cite{shapley} (even though the value-function landscape may not be convex-concave, which implies that Von Neumann's celebrated minimax theorem may not be applicable). \par

\paragraph{Model and Informal Statement of Results.} While the previous works enhance our understanding in \emph{competitive} interactions, i.e., interactions in which gains can only come at the expense of others, MARL in \emph{cooperative} settings remains largely under-explored and constitutes one of the current frontiers in AI research \cite{Daf20,Daf21}. Based on the above, our work is motivated by the following natural question:

\vspace{-0.2cm}
\begin{center}
\textit{Can we get (provable) convergence guarantees for multi-agent RL settings in which cooperation is desirable?}
\end{center}

To address this question, we define and study a class of $n$-agent MDPs that naturally generalize normal form potential games \cite{Mon96}, called \emph{Markov Potential Games (MPGs)}. In words, a multi-agent MDP is a MPG as long as there exists a (state-dependent) real-valued potential function $\Phi$ so that if an agent $i$ changes their policy (and the rest of the agents keep their policy unchanged), the difference in agent $i$'s value/utility, $V^i$, is captured by the difference in the value of $\Phi$ (see \Cref{def:potential}). Weighted and ordinal MPGs are defined similar to their normal form counterparts (see \Cref{rem:owMPG}).\par
Under our definition, we answer the above motivating question in the affirmative. In particular, we show that if every agent $i$ independently runs (with simultaneous updates) policy gradient on his utility/value $V^i$, then, after $O(1/\epsilon^2)$ iterations, the system will reach an $\epsilon$-approximate Nash policy (see informal Theorem \ref{thm:main} and formal \Cref{thm:mainformal}). Moreover, for the finite sample analogue, i.e., if every agent $i$ independently runs (with simultaneous updates) stochastic policy gradient, we show that the system will reach an $\epsilon$-approximate Nash policy after $O(1/\epsilon^5)$ iterations.\par
Along the way, we prove several properties about the structure of MPGs and their Nash policies (see Theorem \ref{thm:main2} and Section \ref{sec:characterization}). In sum, our results can be summarized in the following two Theorems.
\begin{theorem}[Convergence of Policy Gradient (Informal)]\label{thm:main}
Consider a MPG with $n$ agents and let $\epsilon>0$. Suppose that each agent $i$ runs independent policy gradient using direct parameterization on their policy and that the updates are simultaneous. Then, the learning dynamics reach an $\epsilon$-Nash policy after $\mathcal{O}(1/\epsilon^2)$ iterations. If instead, each agent $i$ runs stochastic policy gradient using greedy parameterization (see \eqref{eq:greedyparam}) on his policy and that the updates are simultaneous, then the learning dynamics reach an $\epsilon$-Nash policy after $\mathcal{O}(1/\epsilon^5)$ iterations.
\end{theorem}
\begin{remark}[Improving to $\mathcal{O}(1/\epsilon^5)$]
As far as our result for stochastic policy gradient is concerned, the proof utilizes the auxiliary Lemma \ref{lem:approxsmooth22} about $\alpha$-greedy parametrization (in which eventually we set $\alpha=\epsilon$). In a previous version of the paper, we had a more loose Lemma \ref{lem:approxsmooth22} and instead we used $\alpha= \epsilon^2$. The previous version claimed that $\mathcal{O}(1/\epsilon^6)$ number of iterations suffice for stochastic policy gradient to reach an $\epsilon$-Nash.
\end{remark}
This result holds trivially for weighted MPGs and asymptotically also for ordinal MPGs, see \Cref{rem:oMPGs}.

\begin{theorem}[Structural Properties of MPGs]\label{thm:main2} The following facts are true for MPGs with $n$-agents:
\begin{itemize}[leftmargin=0.6cm,itemsep=0cm]
\item[a.] There always exists a Nash policy profile $(\pi_1 ^*,\dots,\pi_n ^*)$ so that $\pi_i ^*$ is deterministic for each agent $i$ (see \Cref{thm:determininstic}). 

\item[b.] We can construct MDPs for which each state is a (normal-form) potential game but which are not MPGs. This can be true regardless of whether the whole MDP is competitive or cooperative in nature (see \Cref{ex:zerosum,ex:blackhole}, respectively). On the opposite side, we can construct MDPs that are MPGs but which include states that are purely competitive (i.e., zero-sum games), see \Cref{ex:noteverystate}. 

\item[c.] We provide sufficient conditions so that a MDP is a MPG. These include cases where each state is a (normal-form) potential game and the transition probabilities are not affected by agents actions or the reward functions satisfy certain regularity conditions between different states (see conditions C1 and C2 in \Cref{prop:conditions}).
\end{itemize}
\end{theorem}

\paragraph{Technical Overview.} The first challenge in the proof of \Cref{thm:main} is that multi-agent settings (MPGs) do not satisfy the gradient dominance property, which is an important part in the proof of convergence of policy gradient in single-agent settings \cite{Aga20}. In particular, different Nash policies may yield different value to each agent and as a result, there is not a properly defined notion of value in MPGs (in contrast to zero-sum stochastic games \cite{Das20}). On the positive side, we show that agent-wise (i.e., after fixing the policy of all agents but $i$), the value function, $V^i$, satisfies the gradient dominance property along the direction of $\pi_i$ (policy of agent $i$). This can be leveraged to show that every \emph{(approximate) stationary point} (\Cref{def:fostationary}) of the potential function $\Phi$ is an \emph{(approximate) Nash policy} (\Cref{lem:stationary}). As a result, convergence to an approximate Nash policy is established by first showing that $\Phi$ is smooth and then by applying \emph{Projected Gradient Ascent \eqref{eq:pga}} on $\Phi$. This step uses the rather well-known fact that (PGA) converges to $\epsilon$-stationary points in $O(1/\epsilon^2)$ iterations for smooth functions. As a result, by applying PGA on the potential $\Phi$, one gets an approximate Nash policy. Our convergence result then follows by showing that PGA on the potential function, $\Phi$, generates the same dynamics as if each agent $i$ runs independent PGA on their value function, $V^i$. \par
In the case that agents do not have access to exact gradients, we derive a similar result for finite samples. In this case, we apply \emph{Projected Stochastic Gradient Ascent (PSGA)} on $\Phi$ which (as was the case for PGA) can be shown to be the same as when agents apply PSGA independently on their individual value functions. The key is to get an \textit{unbiased sample} for the gradient of the value functions and prove that it has bounded variance (in terms of the parameters of the MPG). This comes from the discount factor, $\gamma$; in this case, $1-\gamma$ can be interpreted as the probability to terminate the MDP at a particular state (and $\gamma$ to continue). This can be used to show that a trajectory of the MDP is an unbiased sample for the gradient of the value functions. To guarantee that the estimate has bounded variance, we apply the approach of \cite{Das20} which requires that agents perform PSGA with $\alpha$-greedy exploration (see \eqref{eq:greedyparam}). The main idea is that this parameterization stays away from the boundary of the simplex throughout its trajectory.
\par
Concerning our structural results, the main technical challenge is the dependence of state-transitions (in addition to agents' rewards) on agents' actions. Our work in this part is mainly concerned with showing that the class of MPGs can be significantly larger than state based potential games but also that even simple coordination games may fail to satisfy the (exact) MPG property. Finally, concerning the existence of a deterministic Nash policies, the main challenge is (as in \Cref{thm:main}) the lack of a (unique) value in general multi-agent settings. As we show in the proof of \Cref{thm:determininstic}, this issue can be still handled within the class of MPGs by constructing single-agent deviations (to deterministic optimal policies) which keep the value of the potential constant (at its global maximum). This process (which leads to a deterministic Nash policy profile) depends critically on the MPG property and does not generalize to arbitrary MARL settings.

\paragraph{Other works on MPGs.} There are only a few papers in the recent literature that define and analyze MARL settings under the term \emph{Markov Potential Games} using slight different definitions (see \cite{Mar12,Val18}). These papers mainly focus on state-based potential MDPs (i.e., MDPs in which every state is a potential game) and require rather restrictive additional conditions, such as equality or monotonicity of the state-based potential functions, to address the computational challenge of finding Nash policies.\footnote{The relation of these conditions to the current work is discussed in more detail in \Cref{prop:conditions} and \Cref{rem:conditions}.} Our current results demonstrate the efficiency of simultaneous policy gradient as a to powerful method to find Nash policies even without additional restrictive assumptions on the state-based potential functions. Moreover, as mentioned in \Cref{thm:main2}, the current definition also encompasses MDPs that are not necessarily potential at each state. To the best of our knowledge, the only (cooperative) MPGs that have been successfully addressed prior to this work, are the ones in which all agents receive the same value/utility \cite{Wan02} and which constitute a subclass of the MPG setting considered in this paper.

%% file: model.tex
\section{Preliminaries}\label{sec:notation}
\paragraph{Markov Decision Process (MDP).} The following notation is standard and largely follows \cite{Aga20} and \cite{Das20}. We consider a setting with $n$ agents who repeatedly select actions in a shared Markov Decision Process (MDP). The goal of each agent is to maximize their respective value function. Formally, a MDP is defined as a tuple $\G = (\S, \N, \{\A_i,R_i\}_{i \in \N}, P, \gamma, \rho)$, where 
\begin{itemize}[leftmargin=*, itemsep=0cm]
    \item $\S$ is a finite state space of size $S=|\S|$. We will write $\Delta(\S)$ to denote the set of all probability distributions over the set $\S$.
    \item $\N=\{1,2,\dots,n\}$ is the set of the $n\ge2$ agents in the game.
    \item $\A_i$ is a finite action space for agent $i\in \N$ with generic element $a_i\in \A_i$. Using common conventions, we will write $\A=\prod_{i\in \N}\A_i$ and $\A_{-i}=\prod_{j\neq i}\A_j$ to denote the joint action spaces of all agents and of all agents other than $i$ with generic elements $\mathbf{a}=(a_i)_{i\in \N}$ and $\mathbf{a_{-i}}=(a_j)_{i\neq j\in \N}$, respectively. According to this notation, we have that $\a=(a_i,\mathbf{a_{-i}})$. We will write $X=|\X|$ and $\Delta(\X)$ to denote the size of any set $\X\in \{\A_i,\A_{-i},\A\}$ and the space of all probability distributions over $\X$, respectively.
    \item $R_i: \mathcal{S} \times \mathcal{A} \to [-1,1]$ is the individual reward function of agent $i\in \N$, i.e., $R_i(s,a_i,\a_{-i})$ is the instantaneous reward of agent $i$ when agent $i$ takes action $a_i$ and all other agents take actions $\a_{-i}$ at state $s\in \S$.
    \item $P$ is the transition probability function, for which $P(s'\mid s,\mathbf{a})$ is the probability of transitioning from $s$ to $s'$ when $\mathbf{a} \in \mathcal{A}$ is the action profile chosen by the agents.
    \item $\gamma$ is a discount factor for future rewards of the MDP, shared by all agents.
    \item $\rho\in \Delta(\S)$ is the distribution for the initial state at time $t=0$.
    \end{itemize}

Whenever time is relevant, we will index the above terms with $t$. In particular, at each time step $t\ge0$, all agents observe the state $s_t\in\S$, select actions $\a_t=(a_{i,t},\a_{-i,t})$, receive rewards $r_{i,t}:=R_i(s_t,\a_t), i\in \N$ and transition to the next state $s_{t+1}\sim P(\cdot \mid s_t,\a_t)$. We will write $\tau=(s_t,\a_t,\r_t)_{t\ge0}$ to denote the trajectories of the system, where $\r_t:=(r_{i,t}),i\in \N$.

\paragraph{Policies and Value Functions.} For each agent $i\in \N$, a deterministic, stationary policy $\pi_i: \S \to \A_i$ specifies the action of agent $i$ at each state $s\in \S$, i.e., $\pi_i(s)=a_i\in \A_i$ for each $s \in \S$. A stochastic, stationary policy $\pi_i: \S \to \Delta(\A_i)$ specifies a probability distribution over the actions of agent $i$ for each state $s\in\S$.  In this case, we will write $a_i\sim\pi_i(\cdot\mid s)$ to denote the randomized action of agent $i$ at state $s\in \S$. We will write $\pi_i\in \Pi_i:=\Delta(\A_i)^S$ and $\pi=(\pi_i)_{i\in \N}\in \Pi:=\times_{i\in \N}\Delta(\A_i)^S$, $\pi_{-i}=(\pi_{j})_{i\neq j\in \N}\in\Pi_{-i}:=\times_{i\neq j\in \N}\Delta(\A_j)^S$ to denote the joint policies of all agents and of all agents other than $i$, respectively. A joint policy $\pi$ induces a distribution $\Pr^{\pi}$ over trajectories $\tau=(s_t,\a_t,\r_t)_{t\ge0}$, where $s_0$ is drawn from the initial state distribution $\rho$ and $a_{i,t}$ is drawn from $\pi_i(\cdot\mid s_t)$ for all $i\in \N$.\par
The value function, $V_s^i:\Pi\to\mathbb R$, gives the expected reward of agent $i\in \N$ when $s_0=s$ and the agents draw their actions, $\a_t=(a_{i,t}, \a_{-i,t})$, at time $t\ge0$ from policies $\pi=(\pi_i,\pi_{-i})$
\begin{equation}\label{eq:value_function}
V^i_s(\pi) := \ex_{\pi} \left[\sum_{t=0}^\infty \gamma^t r_{i,t}\mid s_0=s\right].  
\end{equation}
We also denote $V^i_{\rho}(\pi) = \ex_{s \sim \rho}\left[V^i_s(\pi)\right]$ if the initial state is random and follows distribution $\rho.$
\paragraph{Nash Policies.} The solution concept that will be focusing on is the Nash Policy. Formally:

\begin{definition}[Nash Policy]\label{def:optimalpolicy}
A joint policy, $\pi^*=(\pi_i^*)_{i\in \N}\in \Pi$, is a Nash policy if for each agent $i\in \N$ it holds that
\[V_s^i(\pi_i^*,\pi_{-i}^*)\ge V_s^i(\pi_i,\pi_{-i}^*), \;\; \text{for all } \pi_i\in \Delta(\A_i)^S, \; \text{and all } s\in \S, \]
i.e., if the policy, $\pi_i^*$, of each agent $i\in \N$ maximizes agent $i$'s value function for each starting state $s\in \S$ given the policies, $\pi^*_{-i}=(\pi^*_{j})_{j\neq i}$, of all other agents $j\neq i\in \N$. Similarly, a joint policy $\pi^*=(\pi_i^*)_{i\in \N}$ is an $\epsilon$-Nash policy if there exists an $\epsilon>0$ so that for each agent $i$
\[V_s^i(\pi_i^*,\pi_{-i}^*)\ge V_s^i(\pi_i,\pi_{-i}^*)-\epsilon, \;\; \text{for all } \pi_i\in \Delta(\A_i)^S, \; \text{and all } s\in \S. \]
\end{definition}

\noindent We note that the definition of Nash policy is the same if $s\sim \rho$ (random starting state).
\paragraph{Markov Potential Games.}
We are ready to define the class of MDPs that we will focus on for the rest of the paper, i.e., Markov Potential Games.

\begin{definition}[Markov Potential Game]\label{def:potential}
A Markov Decision Process (MDP), $\G$, is called a \emph{Markov Potential Game (MPG)} if there exists a (state-dependent) function $\Phi_s:\Pi \to \mathbb{R}$ for $s\in\S$ so that 
\begin{align*}
    \Phi_s(\pi_i, \pi_{-i}) - \Phi_s(\pi_i',\pi_{-i}) = V_s^i(\pi_i, \pi_{-i}) - V_s^i(\pi_i', \pi_{-i}),
\end{align*}
for all agents $i\in \N$, all states $s\in\S$ and all policies $\pi_i,\pi_i'\in \Pi_i,\pi_{-i}\in \Pi_{-i}$. By linearity of expectation, it follows that $ \Phi_{\rho}(\pi_i, \pi_{-i}) - \Phi_{\rho}(\pi_i',\pi_{-i}) = V_{\rho}^i(\pi_i, \pi_{-i}) - V_{\rho}^i(\pi_i', \pi_{-i}),$ where $\Phi_{\rho}(\pi) := \ex_{s \sim \rho}\left[\Phi_s(\pi)\right].$
\end{definition}

As in normal-form games, an immediate consequence of this definition is that the value function of each agent in a MPG can be written as a sum of the potential \emph{(common term)} and a term that does not depend on that agent's policy \emph{(dummy term)}, cf. \Cref{prop:separability} in \Cref{app:omitted}. In symbols, for each agent $i\in \N$ there exists a function $U^i_s:\Pi_{-i}\to \mathbb R$ so that \[V_s^i(\pi)=\Phi_s(\pi)+U^i_s(\pi_{-i}), \;\text{ for all }\pi \in \Pi.\]
\begin{remark}[Ordinal and Weighted Potential Games]\label{rem:owMPG}
Similar to normal-form games, one may also define more general notions of MPGs, such as ordinal or weighted Markov Potential Games. Specifically, if for all agents $i\in \N$, all states $s\in\S$ and all policies $\pi_i,\pi_i'\in \Pi_i,\pi_{-i}\in \Pi_{-i}$, the function $\Phi_s, s\in \S$ satisfies
\begin{align*}
    \Phi_s(\pi_i, \pi_{-i}) - \Phi_s(\pi_i',\pi_{-i}) >0 \iff V_s^i(\pi_i, \pi_{-i}) - V_s^i(\pi_i', \pi_{-i}) > 0,
\end{align*}
then the MDP, $\G$, is called an \emph{Ordinal Markov Potential Game (OMPG)}. If there exist positive constants $w_i>0, i \in \N$ so that 
\begin{align*}
    \Phi_s(\pi_i, \pi_{-i}) - \Phi_s(\pi_i',\pi_{-i}) = w_i(V_s^i(\pi_i, \pi_{-i}) - V_s^i(\pi_i', \pi_{-i})),
\end{align*}
then $\G$ is called a \emph{Weighted Markov Potential Game (WMPG)}.

\end{remark}

Similarly to normal-form games, such classes are naturally motivated also in the setting of multi-agent MDPs. As \Cref{ex:blackhole} in \Cref{app:examples} shows, even simple potential-like settings, i.e., settings in which coordination is desirable for all agents, may fail to be exact MPGs (but may still be ordinal or weighted MPGs) due to the dependence of both the rewards and the transitions on agents' decisions. From our current perspective, ordinal and weighted MPGs (as defined in \Cref{rem:owMPG}) remain relevant, since as we argue, policy gradient still converges to Nash policies in these classes of games (see Remark \ref{rem:oMPGs}).

\paragraph{Independent Policy Gradient and Direct Parameterization} We assume that all agents update their policies \emph{independently} according to the \emph{projected gradient ascent (PGA)} or \emph{policy gradient} algorithm on their policies. Independence here refers to the fact that (PGA) requires only local information (each agent's own rewards, actions and view of the environment) to form the updates, i.e., to estimate that agent's policy gradients. Such protocols are naturally motivated and particularly suitable for distributed AI settings in which all information about the interacting agents, the type of interaction and the agent's actions (policies) is encoded in the environment of each agent.\footnote{In practice, even though every agent treats their environment as fixed, the environment changes as other agents update their policies. This is what makes the analysis of such protocols particularly challenging in full generality. It also highlights the importance of studying classes of games (MDPs) in which convergence of independent learning protocols can be obtained such as zero-sum stochastic games \cite{Das20} or MPGs as we do in this paper.} \par
The PGA algorithm is given by 
\begin{equation}\label{eq:pga}
\pi_i^{(t+1)}:=P_{\Delta(\A_i)^S}\<\pi_i^{(t)}+\eta \nabla_{\pi_i} V^i_{\rho}(\pi^{(t)})\>, \tag{PGA}
\end{equation}
for each agent $i\in \N$, where $P_{\Delta(\A_i)^S}$ is the projection onto $\Delta(\A_i)^S$ in the Euclidean norm. We also assume that all players $i\in \N$ use direct policy parameterizations, i.e., 
\begin{equation}\label{eq:parameter}
\pi_{i}(a\mid s)=x_{i,s,a}
\end{equation}
with $x_{i,s,a}\ge 0$ for all $s\in \S,a \in \A_i$ and $\sum_{a\in \A_i}x_{i,s,a}=1$ for all $s\in\S$. This parameterization is complete in the sense that any stochastic policy can be represented in this class \cite{Aga20}. 
\par In practice, agents use \emph{projected stochastic gradient ascent} (PSGA), according to which, the actual gradient, $\nabla_{\pi_i}V_\rho^i(\pi^{(t)})$, is replaced by an estimate thereof that is calculated from a randomly selected (yet finite) sample of trajectories of the MDP. This estimate, $\hat{\nabla}_{\pi_i}^{(t)}$ may be derived from a single or a batch of observations which in expectation behave as the actual gradient. We choose the estimate of the gradient of $V^i_{\rho}$ to be 
\begin{equation}\label{eq:estimator}
\hat{\nabla}_{\pi_i}^{(t)} = R^{(T,t)}_i\sum_{k=0}^T \nabla \log \pi_i (a_k^{(t)} \mid  s_k^{(t)}),
\end{equation}
where $s_0^{t} \sim \rho$, and $R^{(T,t)}_i = \sum_{k=0}^T r^{k}_{i,t}$ is the sum of rewards of agent $i$ for a batch of time horizon $T$ along the trajectory generated by the stochastic gradient ascent algorithm at its $t$-th iterate (recall that the discount factor, $\gamma$, functions as the probability to continue at each step, so $T$ is sampled from a geometric distribution).\par

The direct parameterization is not sufficient to ensure that the variance of the gradient estimator is bounded (as policies approach the boundary). In this case, we will require that each agent $i\in \N$ uses instead direct parameterization with $\alpha$-greedy exploration as follows
\begin{equation}\label{eq:greedyparam}
\pi_{i}(a\mid s)=(1-\alpha)x_{i,s,a}+\alpha/A_i,
\end{equation}
where $\alpha$ is the exploration parameter for all agents.
Under $\alpha$-greedy exploration, it can be shown that (\ref{eq:estimator}) is unbiased and has bounded variance (see Lemma \ref{lem:unbiased}). The form of PSGA is
\begin{equation}\label{eq:psga}
\pi_i^{(t+1)}:=P_{\Delta(\A_i)^S}\<\pi_i^{(t)}+\eta \hat{\nabla}_{\pi_i}^{(t)}\>. \tag{PSGA}
\end{equation}

%% file: properties.tex
\section{Structural Properties of Markov Potential Games}\label{sec:characterization}

\paragraph{Existence of Deterministic Nash Policies.}
The first question that we examine, is whether MPGs possess a \emph{deterministic} Nash policy profile, as is the case in normal-form potential games \cite{Mon96}. In \Cref{thm:determininstic}, we show that this important property indeed carries over (which settles part (a) of informal Theorem \ref{thm:main2}). 
\begin{theorem}[Deterministic Nash Policy Profile]\label{thm:determininstic}
Let $\G$ be a Markov Potential Game (MPG). Then, there exists a Nash policy $\pi^*\in \Delta(\A)^S$ which is deterministic, i.e., for each agent $i\in \N$ and each state $s\in \S$, there exists an action $a_i\in \A_i$ so that $\pi_i^*(a_i\mid s)=1$.
\end{theorem}

The proof of \Cref{thm:determininstic} (which is deferred to \Cref{app:omitted}) exploits the fact that we can iteratively reduce the non-deterministic components of an arbitrary Nash policy profile that corresponds to a global maximizer of the potential and still retain the Nash profile property at all times. At each iteration, we isolate an agent $i\in \N$, and find a deterministic (optimal) policy for that agent in the (single-agent) MDP in which the policies of all other agents but $i$ remain fixed. The important observation is that the resulting profile is again a global maximizer of the potential and hence, a Nash policy profile. This argument critically relies on the MPG structure and does not seem directly generalizable to MDPs that do not satisfy \Cref{def:potential}.

\paragraph{Sufficient Conditions for MPGs.} We next turn to the question of which types of games are captured by \Cref{def:potential}. It is tempting to think that MDPs which are potential at every state (meaning that the immediate rewards at every state are captured by a (normal-form) potential game at that state) are trivially MPGs. As we show in \Cref{ex:zerosum,ex:blackhole}, this intuition fails in the most straightforward way: we can construct simple MDPs that are potential at every state but which are purely competitive (do not possess a deterministic Nash policy) overall (\Cref{ex:zerosum}) or which are cooperative in nature overall but which do not possess an exact potential function (\Cref{ex:blackhole}). 

\begin{example}\label{ex:zerosum}
Consider the two-agent, two-state, and two actions per state MDP, $\G=\<\S=\{0,1\},\N=\{A,B\}, (\A_i=\{0,1\},R_i)_{i\in \N}, P,\rho\>$ in \Cref{fig:zerosum}. At state $0$ ($1$), agent A always receives $+2$ ($0$) and agent B always receives $0$ ($+2$) regardless of the actions they choose. That is, the reward functions for both states are constant, which implies that both states are potential games. The transitions are determinstic and are given by \[s_{t+1} = s_t \oplus a^{s_t}_{A} \oplus a^{s_t}_{B},\] 
where $\oplus$ denotes the $\texttt{xor}$ operator or equivalently, addition modulo $2$, i.e., $1\oplus 1=0$. The MDP $\G$ is illustrated in \Cref{fig:zerosum}.\par

To show that $\G$ is not a MPG, it suffices to show that it cannot have a deterministic Nash policy as should be the case according to \Cref{thm:determininstic}. To obtain a contradiction, assume that agent $A$ is using a deterministic action $a_A^0\in\{0,1\}$ at state $0$. Then, agent $B$, who prefers to move to state $1$, will optimize their utility by choosing the action $a^0_{B}\in\{0,1\}$ that yields $a^0_A\oplus a_B^0=1$. In other words, given any deterministic action of agent $A$ at state $0$, agent $B$ can choose an action that always moves the sequence of play to state $1$. Thus, such an action cannot be optimal for agent $A$ which implies that the MDP $\G$ does not have a deterministic Nash policy profile as claimed. \par
\end{example}
Intuitively, the two agents in \Cref{ex:zerosum} play a game of \emph{matching pennies} in terms of the actions that they choose (since they prefer playing in opposing states). Thus, competition arises due to the opposing preferences of the agents over states even though the immediate rewards at each states are determined by normal form potential games. \par
\Cref{ex:blackhole} shows that 
a state-based potential game may fail to be a MPG even if agents have similar preferences over states. In that case, the reason is that one cannot find an exact potential function due to the dependence of the transitions on agents' actions. However, in the case of \Cref{ex:blackhole}, it is straightforward to show that the game is an ordinal potential game, cf. \Cref{rem:conditions}.

\begin{example}\label{ex:blackhole}
Consider the two-agent, two-state MDP, $\G=(\S=\{0,1\}, \N=\{A, B\}, \{\A_i,$ $R_i\}_{i\in \N}, P, \rho)$ in \Cref{fig:simple}. At state $s_0$, each agent has two actions, $\A_i=\{0,1\}$, whereas at state $s_1$, each agent has a single action. The transitions and instantaneous rewards, $(R_A(s,\a),R_B(s,\a)), s=0,1, \a=(a^s_A,a^s_B)$ of this MDP are shown in \Cref{fig:simple}. If the action profile $\a=(a^0_A,a^0_B)=(0,0)$ is selected at state $s_0$, then the play remains there, otherwise the play transitions to state $s_1$ and remains there forever. \par

\begin{figure}[!tb]
\centering
\vspace{-1cm}
\begin{minipage}[b]{.485\textwidth}
\centering
\begin{scaletikzpicturetowidth}{\linewidth}
\begin{tikzpicture}[state/.style={circle, draw,  minimum size=1.1cm}, scale=\tikzscale, every node/.append style={transform shape}]
\node[state,label={[shift={(-0.5,-2.9)}]$\bordermatrix{
~ & 0 & 1 \cr
0 & 2,0 & 2,0\cr
1 & 2,0 & 2,0\cr}$}] (s1) {$s_0$};
\node[state,align = center,right = 2.5cm of s1,label={[shift={(0.5,-2.9)}]$\bordermatrix{
~ & 0 & 1 \cr
0 & 0,2 & 0,2\cr
1 & 0,2 & 0,2\cr}$}] (bh) {$s_1$};
\draw (s1) edge[loop left, out=200, in =120, looseness=7] node [pos=0.9,above,yshift=2mm,xshift=-1mm] {$a^0_A\oplus a^0_B=0$}  (s1);
\draw (s1) edge[bend left] node [midway] {\footnotesize otherwise} (bh);
\draw (bh) edge[bend left] node [midway] {\footnotesize otherwise} (s1);
\draw (bh) edge[loop left, out=60, in =-20, looseness=7] node [pos=0.1,above,yshift=2mm,xshift=1mm] {$a^1_A\oplus a^1_B=0$}  (bh);
\end{tikzpicture}
\end{scaletikzpicturetowidth}
\caption{A MDP which is potential at every state but which is not a MPG due to conflicting preferences over states. The agents' instantaneous rewards, $(R_A(s,\a),R_B(s,\a))$, are in matrix form below each state $s=0,1$.}
\label{fig:zerosum}
\end{minipage}
\hfill
\begin{minipage}[b]{.485\textwidth}
\centering
\hspace*{0.4cm}\begin{scaletikzpicturetowidth}{0.99\linewidth}
\begin{tikzpicture}[state/.style={circle, draw,  minimum size=1.1cm}, scale=\tikzscale, every node/.append style={transform shape}]
\node[state,label={[shift={(-0.5,-2.8)}]$\bordermatrix{
~ & 0 & 1 \cr
0 & \phantom{-}5,\phantom{-}2 & -1,-2\cr
1 & -5,-4 & \phantom{-}1,\phantom{-}4\cr}$}] (s1) {$s_0$};
\node[state, right = 2.5cm of s1,label={[shift={(0,-2)}]$(0,0)$}] (bh) {$s_1$};
\draw (s1) edge[loop left, out=200, in =120, looseness=7] node [pos=0.9, above, yshift=2mm,xshift=-1mm] { $(a_A^0,a_B^0)=(0,0)$}  (s1);
\draw (s1) edge node [midway, fill=white] {\footnotesize otherwise} (bh);
\draw (bh) edge[loop left, out=60, in =-20, looseness=7] (bh);
\end{tikzpicture}
\end{scaletikzpicturetowidth}
\caption{A MDP which is potential at every state and cooperative in nature but which is not a MPG. The action-dependent transitions do not allow the derivation of an exact potential function.}
\label{fig:simple}
\end{minipage}
\end{figure}
Since, the game at $s_0$ is a version of the Battle of the Sexes and hence a potential game (see also \Cref{app:examples}), there exists a potential function $\phi_0$, such that we may write the instantaneous reward, $R_i(s_0,\a)$ of agent $i=A,B$ at that state as the sum of the potential, $\phi_0(\pi)$ (common term) and a dummy term, $u^i_0(\pi_{-i})$, which does not depend on the action (policy) of agent $i$, but only on the action (policy) of agent $-i$, i.e., $R_i(s_0,\pi)=\phi_0(\pi)+u^i_0(\pi_{-i})$, for $i=A,B$. Here we are using the slight abuse of notation that $R_i(s,\pi) = \ex_{\a \sim \pi} R_i(s,\a)$. This leads (after some standard algebra) to the following expression for the value function $V_0^i(\pi)$ of agent $i=A,B$ with starting state $s_0$
\begin{align*}
    V_0^i(\pi)&=\frac{\phi_0(\pi)}{1-\gamma pq}+\frac{u_i(\pi_{-i})}{1-\gamma pq}, \quad \text{for } i=A,B,
\end{align*}
where $p,q\in[0,1]$ are the probabilities with which agents $A$ and $B$ respectively select their action $0$ at state $s_0$. This expression clearly indicates the complexity that emerges in MPGs versus static games. Namely, the first term of the value function is a common term (same for both agents) that can conveniently become part of a potential function. However, the second term is a mixture of a common term (denominator) and a term that is different for each agent (numerator). The reason is that the policy of each agent determines the time that the agents spend at each state and thus, it does not (generally) allow for an agent independent term (as required by the definition of a potential game). However, this game is clearly a \emph{potential-like game} in which agents have common interests. This motivates to look at the notion of ordinal or weighted MPGs. Note that (by a straightforward calculation) this game is an ordinal MPG for the potential function $\Phi_s=\phi_s$ for $s=0,1$.
\end{example}

Based on the intuition from the previous Examples, we formulate the following sufficient conditions in \Cref{prop:conditions} which ensure that a state based potential game (i.e., a game that is potential at every state) is also a MPG according to \Cref{def:potential} (cf. \Cref{thm:main2} part (c)).

\begin{proposition}[Sufficient Conditions for MPGs]\label{prop:conditions}
Consider a MDP $\G$ in which every state $s \in \S$ is a potential game, i.e., the immediate rewards $R(s,\mathbf{a})=(R_i(s,\a))_{i\in\N}$ for each state $s\in \S$ are captured by the utilities of a (normal-form) potential game with potential function $\phi_s$. Additionally, assume that one of the following conditions holds
\begin{itemize}[leftmargin=0.7cm,noitemsep]
\item[C1.] Agent-Independent Transitions: $P(s'\mid s,\a)$ does not depend on $\mathbf{a}$, that is, $P(s'\mid s,\a) = P(s'\mid s)$ is just a function of the present state for all states $s,s'\in \S$.
\item[C2.] Equality of Individual Dummy Terms: $P(s'\mid s,\a)$ is arbitrary but the dummy terms of each agent's immediate rewards are equal across all states, i.e., there exists a function $u_s^i:\A_{-i}\to \mathbb R$ such that 
$R_i(s,a_i,\a_{-i})=\phi_s(a_i,\a_{-i})+u_s^i(\a_{-i})$, and \[ \nabla_{\pi_i(s)}\ex_{\tau\sim\pi}\lt\sum\nolimits_{t=0}^{\infty}\gamma^t u_{s_t}^i(\a_{-i,t})\mid s_0=s' \rt=c_s\mathbf{1},\]  for all states $s',s\in \S$, where $c_s\in\mathbb R$ and $\mathbf{1}\in \mathbb{R}^{A_i}$, where $\pi_i(s)$ corresponds to the policy distribution of agent $i$ at state $s$.
\end{itemize}
If either C1 or C2 are true, then $\G$ is a MPG.
\end{proposition}

\begin{remark}\label{rem:conditions}
As the rest of the proofs of \Cref{sec:characterization}, the proof of \Cref{prop:conditions} is provided in \Cref{app:omitted}. The following remarks are due.   
\begin{enumerate}[leftmargin=0.6cm,itemsep=0cm]
\item Condition C1 can also be viewed as a special case of condition C2. However, due to its simplicity, it is more instructive to state C1 separately. Condition C2 (or variations of it) are already present in existing studies of potential-like MDPs \cite{Mar12,Val18}. \Cref{ex:blackhole} shows that such conditions are restrictive, in the sense that they do not capture very simple MDPs that intuitively have a potential-like (cooperative) structure. This motivates the study of ordinal or weighted potential games as natural models to capture such cases. As we show, our convergence results about independent policy gradient naturally extend to these classes as well (see \Cref{rem:oMPGs}). 
\item Condition C2 is (trivially) satisfied when $u^i_s$ does not depend on the state $s$ nor on the actions of other agents, i.e., $u^i_s(\a_{-i})\equiv c^i$ for some constant $c^i \in \mathbb R$ for all $\a_{-i}\in \A_{-i}$ and all $s\in \S$. A special case is provided by MDPs in which the instantaneous rewards of all agents are the same at each state, i.e., such that $R_i(s,a_i,\a_{-i})=\phi_s(a_i,\a_{-i})$ for all agents $i\in \N$, all actions $a_i\in \A_i$ and all states $s\in \S$. MDPs that satisfy this condition form a subclass of the current definition and have been studied under the name \emph{Team Markov Games} in \cite{Wan02}. 
\end{enumerate}
\end{remark}

The previous discussion focuses on games that are potential at every state as natural candidates to generalize the notion of normal-form games to state games. This leaves an important question unanswered: are there games which are not potential at every state but which are captured by the our current definition of MPGs? \Cref{ex:noteverystate} answers this question in the affirmative. Together with \Cref{ex:zerosum}, this settles the claim in \Cref{thm:main2}, part (b).

\begin{example}[Not potential at every state may still imply MPG]\label{ex:noteverystate}
Consider the 2-agent MDP of \Cref{fig:example}. 
\begin{figure}[!tb]\centering
\begin{scaletikzpicturetowidth}{0.7\linewidth}
\vspace{-2.5cm}
\begin{tikzpicture}[state/.style={circle, draw,  minimum size=1.1cm}, scale=\tikzscale, every node/.append style={transform shape}]
\node[state, label={[shift={(-1,-2.9)}]$\bordermatrix{
~ & H & T \cr
H & \phantom{-}1,-1 & -1,\phantom{-}1\cr
T & -1,\phantom{-}1 & \phantom{-}1,-1\cr}$}]			(s0) {$s_0$};
\node[label={[shift={(0,-2.85)}]$\<\frac{1}{\gamma},-\frac{1}{\gamma}\>$},state, right = 1cm of s0] (sHT) {$s_{HH}$};
\node[label={[shift={(0,-2.85)}]$\<-\frac{1}{\gamma},\frac{1}{\gamma}\>$}, state, right=0.9cm of sHT] (sHH) {$s_{HT}$};
\node[label={[shift={(0,-2.85)}]$\<-\frac{1}{\gamma},\frac{1}{\gamma}\>$}, state, right =0.9cm of sHH] (sTH) {$s_{TH}$};
\node[label={[shift={(0,-2.85)}]$\<\frac{1}{\gamma},-\frac{1}{\gamma}\>$}, state, right=0.9cm of sTH] (sTT) {$s_{TT}$};
\node[state, right = 2.5cm of sTT, label={[shift={(0,-2.9)}]$\bordermatrix{
~ & L & R \cr
L & (1,1) & (9,0)\cr
R & (0,9) & (6,6)\cr}$}]			(s1) {$s_1$};
\draw[->] (s1) edge [bend right=35,shorten <=2pt]  node[midway,above] {$p_0$} (s0); 
\draw[every loop,bend right = 30,shorten <=2pt] 
(s0) edge (sHH)
(s0) edge (sHT)
(s0) edge (sTT)
(s0) edge (sTH);
\draw[every loop,bend left = 30,shorten <=2pt] 
(sHT) edge (s1)
(sTT) edge (s1)
(sHH) edge (s1)
(sTH) edge (s1);
\draw[dotted] ($(s0)+(-3.2,-2.3)$) rectangle (9.3,0.7) node[yshift=5cm]{};
\draw (s1) edge[loop left, out=60, in =-20, looseness=7] node [right,midway] {$1-p_0$} (s1);
\end{tikzpicture}
\end{scaletikzpicturetowidth}
\caption{A $2$-player MPG which is not potential at every state. The rewards in state $s_1$ form a potential game, whereas the rewards in $s_0$ do not. However, the states inside the dotted rectangle do form a potential game and this can be leveraged to show that the whole MPG is a potential game whenever $p_0$ does not depend on agents' actions.}\label{fig:example}
\end{figure}
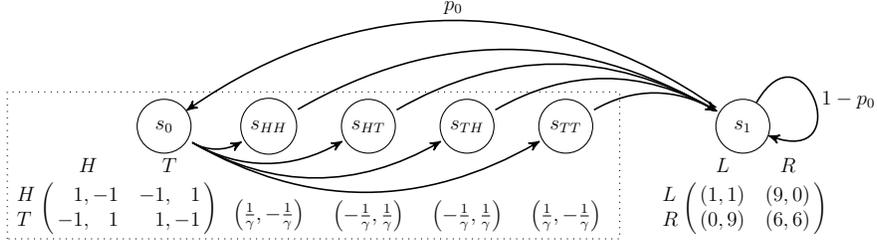
At state $s_0$, agents' rewards, $(R_1(s_0,\a),R_2(s_0,\a))$ form a constant sum (equivalent to zero-sum) game. The agents' actions at $s_0$ induce a deterministic transition to a state $s_{ab}$ with $a,b \in \{H,T\}$ in which the only available actions are precisely the chosen actions at $s_0$. Each agent's instantaneous reward at this state is the reward of the other agent at $s_0$ (scaled by $1/\gamma$). The MDP then transitions deterministically to state $s_1$ which is a potential game with rewards $(R_1(s_1,\a),R_2(s_1,\a))$. After the agents select their actions at $s_1$, there is an exogenous given probability, $p_0$, according to which the play transitions to state $s=0$. Otherwise it remains at $s_1$.\par
While the game at state $s=0$ is not a potential game, the combined states in the dotted rectangle of \Cref{fig:example} do form a potential game, with potential function equal to the sum of the agent's payoffs at $s_0$ (the rewards of both agents are equal for every pass of the play through the states in the dotted rectangle). Thus, it is not hard to see that both value functions are of the form \[V_s^i\<\pi_1,\pi_2\>=c_1\<s\>\cdot \xx_0\<R_1(s_0,\a)+R_2(s_0,\a)\>\yy_0+c_2\<s\>\cdot \xx_1R_i(s_1,\a)\yy_1,\] 
for $s\in\{s_0,s_1\}$ and $i=\{1,2\}$,
where $c_1\<s\>,c_2\<s\>>0$ are appropriate constants that depend only the state $s\in\{s_0,s_1\}$ and not on the agents. Since the game at $s_1$ is a potential game, with potential function given by a $2\times 2$ matrix $\phi_1$, it is immediate to see that  
\[\Phi_s\<\pi_1,\pi_2\>:=c_1\<s\>\xx_0\<R_1(s_0,\a)+R_2(s_0,\a)\>\yy_0+c_2\<s\>\cdot \xx_1\phi_1\yy_1,\text{for } s\in\{s_0,s_1\}\]
is a potential function so that $\G$ satisfies the definition of an MPG.
\end{example}

%% file: convergence.tex
\section{Convergence of Policy Gradient in Markov Potential Games}\label{sec:convergence} 
The current section presents the proof of convergence of (projected) policy gradient to approximate Nash policies in Markov Potential Games (MPGs). We analyze the cases of both infinite and finite samples using direct and $\alpha$-greedy parameterizations, respectively. \par

Before we proceed with the formal statements and proofs of this section, we  provide the commonly used definition of distribution mismatch coefficient \cite{Kak02} applied to our setting.
\begin{definition}[Distribution Mismatch coefficient]
Let $\mu$ be any distribution in $\Delta(\S)$ and let $\mathcal{O}$ be the set of policies $\pi \in \Delta(\mathcal{A})^S$. We call 
\[D :=\max_{\pi \in \mathcal{O}}\norm{\frac{d^{\pi}_{\mu}}{\mu}}_{\infty}\] 
the \emph{distribution mismatch coefficient}, where $d^{\pi}_{\mu}$ is the discounted state distribution (\ref{eq:visitation}). 
\end{definition}

The first auxiliary Lemma has to do with the projection operator that is used on top of the independent policy gradient, so that the policy vector $\pi^{(t)}_i$ remains a probability distribution for all agents $i\in \N$ (see (\ref{eq:pga})). It is not hard to show (due to separability) that the projection operator being applied independently for each agent $i$ on $\Delta(\mathcal{A}_i)^S$ is the same as jointly applying projection on $\Delta(\mathcal{A})^S$. This is the statement of \Cref{claim:projection}.
\begin{lemma}[Projection Operator]\label{claim:projection} Let $\pi := (\pi_1,...,\pi_n)$ be the policy profile for all agents and let
\[\pi' = \pi + \eta \nabla_\pi\Phi_{\rho}(\pi),\] 
be a gradient step on the potential function for a step-size $\alpha>0$. Then, it holds that 
\[P_{\Delta(\A)^S} (\pi') = (P_{\Delta(\A_1)^S} (\pi'_1) , \dots , P_{\Delta(\A_n)^S}(\pi'_n)).\]
\end{lemma}

The main implication of \Cref{claim:projection} along with the equality of the derivatives between value functions and the potential function in MPGs, i.e., $\nabla_{\pi_i} V_s^i(\pi)=\nabla_{\pi_i}\Phi(\pi)$ for all $i\in \N$ (see property P2 in Proposition \ref{prop:separability}), is that running independent \eqref{eq:pga} on each agent's value function is equivalent to running \eqref{eq:pga} on the potential function $\Phi$. In turn, \Cref{lem:stationary} suggests that as long as policy gradient reaches a point $\pi^{(t)}$ with small gradient along the directions in $\Delta(\mathcal{A})^S$, it must be the case that $\pi^{(t)}$ is an approximate Nash policy. Together with \Cref{claim:projection}, this will be sufficient to prove convergence of 
\eqref{eq:pga}.

\paragraph{Stationary point of $\Phi$.} To proceed, we need the formal definition of a stationary point for the potential function $\Phi$ which is given below. \begin{definition}[$\epsilon$-Stationary Point of $\Phi$]\label{def:fostationary} A policy profile $\pi := (\pi_1,...,\pi_n) \in \Delta(\mathcal{A})^S$ is called $\epsilon$-stationary for $\Phi$ w.r.t distribution $\mu$ as long as \begin{equation}     \max_{(\pi_1 + \delta_1,\dots,\pi_n+\delta_n) \in \Delta(\mathcal{A})^S,  \sum_{i \in \mathcal{N}}\norm{\delta_i}^2_2 \leq 1 } \sum_{i \in \mathcal{N}} \delta_i^{\top} \nabla_{\pi_i} \Phi_{\mu}(\pi) \leq \epsilon \end{equation} \end{definition} In words, the function $\Phi(\pi)$ cannot increase in value by more than $\epsilon$ along every possible local direction $\delta=(\delta_1,\dots,\delta_n)$ that is feasible (namely $\pi+\delta$ is also a policy profile).

\begin{lemma}[Stationarity of $\Phi$ implies Nash]\label{lem:stationary} Let $\epsilon \geq 0$, and let $\pi$ be an $\epsilon$-stationary point of $\Phi$ (see \Cref{def:fostationary}). Then, it holds that $\pi$ is a $\frac{\sqrt{S}D\epsilon}{1-\gamma}$-Nash policy.
\end{lemma}

To prove \Cref{lem:stationary}, we will need the Gradient Domination property that has been shown to hold in single-agent MDPs \cite{Aga20}. This is presented in \Cref{lem:gdom}.

\begin{lemma}[Agent-wise Gradient Domination Property in MPGs \cite{Aga20}]\label{lem:gdom}
Let $\G$ be a MPG with potential function $\Phi$, fix any agent $i \in \mathcal{N}$, and let $\pi=(\pi_i,\pi_{-i})\in \Delta(\A)^S$ be a policy. Let $\pi_i^*$ be an optimal policy for agent $i$ in the single agent MDP in which the rest of the agents are fixed to choose $\pi_{-i}.$ Then, for the policy $\pi^*=(\pi^*_i,\pi_{-i}) \in \Delta(\A)^S$ that differs from $\pi$ only in the policy component of agent $i$, it holds that 
\[\Phi_{\rho}(\pi^*)-\Phi_{\rho}(\pi)\le \frac{1}{1-\gamma}\left \|\frac{d^{\pi^*}_\rho}{\mu}\right\|_{\infty}\max_{\pi'=(\pi'_i,\pi_{-i})}(\pi'-\pi)^\top \nabla_{\pi_i}\Phi_{\mu}(\pi),\]
for any distributions $\mu, \rho \in \Delta(\S)$.
\end{lemma}

\begin{remark}[Best Response]\label{rem:bestresponse}
Intuitively, \Cref{lem:gdom} implies that there is a \emph{best response} structure in the agents' updates that we can exploit to show convergence of (projected) policy gradient to a Nash policy profile. In particular, given a fixed policy profile of all agents other than $i$, the decision of agent $i$ is equivalent to the decision of that agent in a single MDP. Thus, the following inequality (which stems directly from the gradient domination property in the single MDP)
\[V^i_s(\pi^*)-V^i_s(\pi)\le \frac{1}{1-\gamma}\left \|\frac{d^{\pi^*}_s}{\mu}\right\|_{\infty}\max_{\pi'=(\pi'_i,\pi_{-i})}(\pi'-\pi)^\top \nabla_{\pi_i}V^i_{\mu}(\pi)
\]
implies that any stationary point of $V_s^i$ (w.r.t the variables $x_{i,s,a}$ of agent's $i$ policy with the rest of the variables being fixed) is an optimal policy for $i$, i.e., a best response given the policies of all other agents.
\end{remark}

\Cref{lem:gdom} also suggests that there is an important difference in the Gradient Domination Property between (multi-agent) MPGs and single agent MDPs (cf. Lemma 4.1 in \cite{Aga20}). Specifically, for MPGs, the value (of each agent) at different Nash policies may not be unique\footnote{This is in contrast to single-agent MDPs for which the agent has a unique optimal value even though their optimal policy may not necessarily be unique.} which implies that the gradient domination property, as stated in \Cref{lem:gdom}, will only be enough to guarantee convergence to \emph{one of} the optimal (stationary) points of $\Phi$ (and not necessarily to the absolute maximum of $\Phi$). Having all these in mind, we can now prove \Cref{lem:stationary}.

\begin{proof}[Proof of \Cref{lem:stationary}]
Fix agent $i$ and suppose that $i$ deviates to an optimal policy $\pi^*_i$ (w.r.t the corresponding single agent MDP). Since $\pi$ is $\epsilon$-stationary it holds that (Definition \ref{def:fostationary}) 
\begin{equation}\label{eq:helpm}
\max_{\pi'_i \in \Delta(\mathcal{A}_i)^S}(\pi'_i - \pi_i)^{\top} \nabla_{\pi_i} \Phi_{\mu}(\pi) \leq \sqrt{S}\epsilon.
\end{equation}
Thus, with $\pi^* = (\pi^*_i,\pi_{-i})$, \Cref{lem:gdom} implies that 
\begin{equation}\label{eq:almost}
\begin{split}
\Phi_{\rho}(\pi^*)-\Phi_{\rho}(\pi)&\le \frac{1}{1-\gamma}\left \|\frac{d^{\pi^*}_\rho}{\mu}\right\|_{\infty}\max_{\pi'=(\pi'_i,\pi_{-i})}(\pi'-\pi)^\top \nabla_{\pi_i}\Phi_{\mu}(\pi)
\\&\stackrel{(\ref{eq:helpm})}{\leq}\frac{D}{1-\gamma}\sqrt{S}\epsilon.
\end{split}
\end{equation}
Thus, using the definition of the potential function (cf. \Cref{def:potential}), we obtain that
\[
V^i_{\rho}(\pi^*)-V^i_{\rho}(\pi) =\Phi_{\rho}(\pi^*)-\Phi_{\rho}(\pi) \leq \frac{\sqrt{S}D \epsilon}{1-\gamma}\,.
\]
Since the choice of $i$ was arbitrary, we conclude that $\pi$ is an $\frac{\sqrt{S}D \epsilon}{1-\gamma}$-approximate Nash policy.
\end{proof}

The last critical step before we proceed to the formal statement and proof of \Cref{thm:main} is that the potential function $\Phi$ is smooth. This fact is used in the analysis of both (\ref{eq:pga}) and its stochastic counterpart (PSGA).

\begin{lemma}[Smoothness of $\Phi$]\label{claim:smoothness} Let $A_{\max} := \max_{i\in\N} |\mathcal{A}_i|$ (the maximum number of actions for some agent).
Then, for any initial state $s_0\in \S$ (and hence for every distribution $\mu\in \Delta(\S)$ on states) it holds that
\begin{equation}\label{eq:smooth}
    \norm{\nabla_{\pi}\Phi_{s_0}(\pi)-\nabla_{\pi}\Phi_{s_0}(\pi')}_2 \leq \frac{2n\gamma A_{\max}}{(1-\gamma)^3} \norm{\pi - \pi'}_2
\end{equation}
i.e., $\Phi_{\mu}(\pi)$ is $\frac{2n\gamma A_{\max}}{(1-\gamma)^3}$-smooth.
\end{lemma}

Importantly, $A_{\max} := \max_{i\in\N} |\mathcal{A}_i|$, i.e., the maximum number of actions for some agent, scales \emph{linearly} in the number of agents.

\paragraph{Exact gradients case.} We are now ready to prove Theorem \ref{thm:main} (restated formally), following standard arguments about \eqref{eq:pga}. Recall that the global maximum among all values/utilities of agents must be at most one. 

\begin{theorem}[Formal \Cref{thm:main}, part (a)]\label{thm:mainformal} Let $\mathcal{G}$ be a MPG and consider an arbitrary initial state. Let also $A_{\max} = \max_i |\mathcal{A}_i|$, and set the number of iterations to be $T=\frac{16n\gamma D^2 S A_{\max} \Phi_{\max}}{(1-\gamma)^5\epsilon^2}$ and the learning rate (step-size) to be $\eta = \frac{(1-\gamma)^3}{2n\gamma A_{\max}}$. If the agents run independent projected policy gradient (\ref{eq:pga}) starting from arbitrarily initialized policies, then there exists a $t\in \{1,\dots,T\}$ such that $\pi^{(t)}$ is an $\epsilon$-approximate Nash policy. 
\end{theorem}
\begin{proof}
The first step is to show that $\Phi$ is a $\beta$-smooth function, in particular, that $\nabla_{\pi}\Phi$ is $\beta$-Lipschitz with $\beta=\frac{2n\gamma A_{\max} }{(1-\gamma)^3}$ as established in \Cref{claim:smoothness}. Then, a standard \emph{ascent lemma} for Gradient Ascent (see \Cref{lem:descent} from \cite{bubeck}) implies that for any $\beta$-smooth function $f$ it holds that $f(x') - f(x) \geq \frac{1}{2\beta} \norm{x'-x}^2_2$ where $x'$ is the next iterate of \eqref{eq:pga}. Applied to our setting, this gives

\begin{equation}\label{eq:ascent}
    \Phi_{\mu}(\pi^{(t+1)}) - \Phi_{\mu}(\pi^{(t)}) \geq \frac{(1-\gamma)^3}{4n\gamma A_{\max} }\norm{\pi^{(t+1)}-\pi^{(t)}}_2^2
\end{equation}
Thus, if the number of iterates, $T$, is $\frac{16n\gamma D^2 S A_{\max}  }{(1-\gamma)^5\epsilon^2}$, then there must exist a $1 \leq t\leq T$ so that $\norm{\pi^{(t+1) }-\pi^{(t)}}_2 \leq \frac{\epsilon(1-\gamma)}{2D\sqrt{S}}$. Using a standard approximation property (see \Cref{lem:approxsmooth}), we then conclude that $\pi^{(t+1)}$ will be a $\frac{\epsilon(1-\gamma)}{D\sqrt{S}}$-stationary point for the potential function $\Phi$. Hence, by \Cref{lem:stationary}, it follows that $\pi^{(t+1)}$ is an $\epsilon$-Nash policy and the proof is complete.
\end{proof}

\paragraph{Finite sample case.}
 In the case of finite samples, we analyze (\ref{eq:psga}) on the value $V^i$ of each agent $i$ which (as was the case for PGA) can be shown to be the same as applying projected gradient ascent on $\Phi$. The key is to get an estimate of the gradient of $\Phi$ (\ref{eq:estimator}) at every iterate. Note that $1-\gamma$ now captures the probability for the MDP to terminate after each round (and it does not play the role of a discounted factor since we consider finite length trajectories). \Cref{lem:unbiased} argues that the estimator of equation \eqref{eq:estimator} is both unbiased and bounded.
 
\begin{lemma}[Unbiased estimator with bounded variance]\label{lem:unbiased}
It holds that $\hat{\nabla}_{\pi_i}^{(t)}$ is an unbiased estimator of $\nabla_{\pi_i} \Phi$ for all $i\in \N$, that is
\[\mathbb{E}_{\pi^{(t)}}\hat{\nabla}_{\pi_i}^{(t)} = \nabla_{\pi_i} \Phi_{\mu}(\pi^{(t)}) \textrm{ for all }i\in \N.\]
Moreover, for all agents $i\in \N$, it holds that 
\[\mathbb{E}_{\pi^{(t)}} \norm{\hat{\nabla}_{\pi_i}^{(t)}}_2^2\leq  \frac{24A_{\max}^2}{\alpha (1-\gamma)^4}, \textrm{ for all } i\in \N.\]
\end{lemma}
\begin{proof}
 It is straightforward from Lemma \ref{lem:unbiaseddas20} and the equality of the partial derivatives between the value functions and the potential, i.e., $\nabla_{\pi_i} \Phi_{\mu} = \nabla_{\pi_i}V^i_{\mu}$ for all $i\in \N$ (see property P2 in \Cref{prop:separability}).
\end{proof}

We now state and prove part (b) of \Cref{thm:main}.
\begin{theorem}[Formal \Cref{thm:main}, part (b)]\label{thm:mainformalb} Let $\mathcal{G}$ be a MPG and consider an arbitrary initial state. Let  $A_{\max} = \max_i |\mathcal{A}_i|$, and set the number of iterations to be $T=\frac{12288 \gamma n^8 A_{\max}^9 D^4 S^4 }{\epsilon^5 (1-\gamma)^{23}}$ and the learning rate (step-size) to be $\eta = \frac{\epsilon^3(1-\gamma)^{15}}{768 n^5 \gamma D^2S^2 A^6_{\max} }$. If the agents run projected stochastic policy gradient (\ref{eq:psga}) starting from arbitrarily initialized policies and using $\alpha$-greedy parametrization with $\alpha = \epsilon$, then  there exists a $t\in \{1,\dots,T\}$ such that in expectation, $\pi^{(t)}$ is an $\epsilon$-approximate Nash policy.  
\end{theorem}

\begin{proof}
Let $\delta_t = \hat{\nabla}_{\pi}^{(t)}- \nabla_{\pi} \Phi_{\mu}(\pi^{(t)})$ and set $\lambda = \frac{(1-\gamma)^3}{2n\gamma A_{\max}}$ (the inverse of the smoothness parameter in \ref{claim:smoothness}). We follow the analysis of Projected Stochastic Gradient Ascent for non-convex smooth-functions (see \cite{DD18}, Theorem 2.1) that makes use of the Moreau envelope. 
Let \[\phi_{\lambda}(x) = \min_{y \in \Delta(\mathcal{A})^S} \left\{-\Phi_{\mu}(y)+\frac{1}{\lambda}\norm{x-y}_2^2\right\},\] (definition of Moreau envelope for our objective $\Phi$). Moreover, we set \\
$y^{(t+1)} = \arg\min_{y \in \Delta(\mathcal{A})^S} \left\{-\Phi_{\mu}(y)+\frac{1}{\lambda}\norm{\pi^{(t)}-y}_2^2\right\}$. From the definition of $\phi$ and a standard property of projection we get 
\begin{equation}\label{eq:firstphi}
\begin{split}
    \phi_{\lambda}(\pi^{(t+1)}) &\leq -\Phi_{\mu}(y^{(t+1)}) + \frac{1}{\lambda}\norm{\pi^{(t+1)}-y^{(t+1)}}_2^2
    \\&\leq -\Phi_{\mu}(y^{(t+1)}) + \frac{1}{\lambda}\norm{\pi^{(t)} +\eta \hat{\nabla}_{\pi}^{(t)} -y^{(t+1)}}_2^2
    \\&= -\Phi_{\mu}(y^{(t+1)}) + \frac{1}{\lambda}\norm{\pi^{(t)}  -y^{(t+1)}}_2^2 + \frac{\eta^2}{\lambda}\norm{\hat{\nabla}_{\pi}^{(t)}}_2^2+\frac{2\eta}{\lambda}(\pi^{(t)}  -y^{(t+1)})^{\top}\hat{\nabla}_{\pi}^{(t)}
\end{split}
\end{equation}

Since $\hat{\nabla}_{\pi}^{(t)}$ is unbiased (Lemma \ref{lem:unbiased}) we have that $\mathbb{E}[\delta_t|\pi^{(t)}]=0$, therefore $\mathbb{E}\left[\delta^{\top}_t(y^{(t+1)}-\pi^{(t)})\right]=0.$ Additionally, by Lemma \ref{lem:unbiased} (applied for all agents $i$) we also have $\mathbb{E}\left[\norm{\hat{\nabla}_{\pi}^{(t)}}_2^2\right] \leq \frac{24nA^2_{\max}}{\alpha (1-\gamma)^4}.$  Hence by taking expectation on (\ref{eq:firstphi}) we have:
\[
\mathbb{E}[\phi_{\lambda}(\pi^{(t+1)})]   \leq \mathbb{E}\left[-\Phi_{\mu}(y^{(t+1)})+ \frac{1}{\lambda}\norm{\pi^{(t)}-y^{(t+1)}}_2^2\right]+\frac{2\eta}{\lambda}\mathbb{E}[(\pi^{(t)} -y^{(t+1)})^{\top}\nabla_{\pi} \Phi_{\mu}(\pi^{(t)})]+\frac{24\eta^2 nA^2_{\max}}{\lambda\alpha (1-\gamma)^4}.
\]
Using the definition of Moreau envelope and the fact that $\Phi$ is $\frac{1}{\lambda}$-smooth (Lemma \ref{claim:smoothness}, after the parametrization, the smoothness parameter does not increase) we conclude that \begin{equation*}
\begin{split}
\mathbb{E}[\phi_{\lambda}(\pi^{(t+1)})]   &\leq \mathbb{E}[\phi_{\lambda}(\pi^{(t)})]+\frac{2\eta}{\lambda}\mathbb{E}[(\pi^{(t)} -y^{(t+1)})^{\top}\nabla_{\pi} \Phi_{\mu}(\pi^{(t)})]+\frac{24\eta^2 nA^2_{\max}}{\lambda\alpha (1-\gamma)^4}
\\&\leq \mathbb{E}[\phi_{\lambda}(\pi^{(t)})]+\frac{2\eta}{\lambda}\mathbb{E}\left[\Phi_{\mu}(\pi^{(t)}) - \Phi_{\mu}(y^{(t+1)}) + \frac{1}{2\lambda}\norm{\pi^{(t)}-y^{(t+1)}}_2^2\right]+\frac{24\eta^2 nA^2_{\max}}{\lambda\alpha (1-\gamma)^4},
\end{split}
\end{equation*}
or equivalently
\begin{equation}\label{eq:lasttrick}
\mathbb{E}[\phi_{\lambda}(\pi^{(t+1)})] - \mathbb{E}[\phi_{\lambda}(\pi^{(t)})] \leq \frac{2\eta}{\lambda}\mathbb{E}\left[\Phi_{\mu}(\pi^{(t)}) - \Phi_{\mu}(y^{(t+1)}) + \frac{1}{2\lambda}\norm{\pi^{(t)}-y^{(t+1)}}_2^2\right]+\frac{24\eta^2 nA^2_{\max}}{\lambda\alpha (1-\gamma)^4}
\end{equation}
    

Adding telescopically (\ref{eq:lasttrick}), dividing by $T$ and because w.l.o.g $-\Phi \in [-1,0]$, we get that
\begin{equation}\label{eq:combine2}
\begin{split}
\frac{1}{T} + \frac{24\eta^2 nA^2_{\max}}{\lambda\alpha (1-\gamma)^4} & \geq \frac{2\eta}{\lambda T}\sum_{t=1}^T\mathbb{E}\left[\Phi_{\mu}(y^{(t+1)}) - \Phi_{\mu}(\pi^{(t)}) \right]-\frac{\eta}{\lambda^2 T}\sum_{t=1}^T\mathbb{E}\left[\norm{y^{(t+1)}-\pi^{(t)}}_2^2\right]
\\& \geq  \min_{t \in [T]} \left\{\frac{2\eta}{\lambda}\mathbb{E}\left[\Phi_{\mu}(y^{(t+1)}) - \Phi_{\mu}(\pi^{(t)}) \right] -\frac{\eta}{\lambda^2}\mathbb{E}[\norm{y^{(t+1)} - \pi^{(t)}}_2^2]\right\}
\end{split}
\end{equation}

Let $t*$ be the time index that minimizes the above. We show the following inequality (which provides a lower bound on the RHS of (\ref{eq:combine2}):
\begin{equation}\label{eq:helpp}
\mathbb{E}\left[\Phi_{\mu}(y^{(t*+1)}) - \Phi_{\mu}(\pi^{(t*)}) \right] -\frac{1}{2\lambda}\mathbb{E}\left[\norm{y^{(t*+1)} - \pi^{(t*)}}_2^2\right] 
\geq \frac{1}{2\lambda}\mathbb{E}\left[\norm{y^{(t*+1)}-\pi^{(t*)}}_2^2\right]\,,
\end{equation}
which follows from observing that  $- \Phi_{\mu}(\pi^{(t*)}) \geq \phi_{\lambda}(\pi^{(t)}) = - \Phi_{\mu}(y^{(t*+1)}) + \frac{1}{\lambda} \norm{y^{(t*+1)} - \pi^{(t*)}}_2^2$ using the definitions of the Moreau envelope and~$y^{(t+1)}.$


Combining (\ref{eq:combine2}) with (\ref{eq:helpp}) we conclude that
\[
    \frac{1}{T} + \frac{24\eta^2 nA^2_{\max}}{\lambda\alpha (1-\gamma)^4} \geq \frac{\eta}{\lambda^2}\mathbb{E}\left[\norm{y^{(t*+1)}-\pi^{(t*)}}_2^2\right].
\]
By Jensen's inequality it occurs that
\begin{equation}\label{eq:final}
\mathbb{E}\left[\norm{y^{(t*+1)}-\pi^{(t*)}}_2\right] \leq \sqrt{\frac{\lambda^2}{\eta T} + \eta \frac{24 n\lambda A_{\max}^2}{\alpha (1-\gamma)^4}}.
\end{equation}
To get an $\epsilon$-Nash policy, we have to bound $\norm{y^{(t*+1)} - \pi^{(t*)}}_2 \leq \frac{\epsilon(1-\gamma)^4}{2D\sqrt{S}(2(1-\gamma)^3+\sqrt{S}(nA_{max})^{3/2})}$ and choose $\alpha = \epsilon$ in the greedy parametrization. This is true because of Lemma \ref{lem:approxsmooth22}, \Cref{lem:stationary} and the fact that the gradient mapping norm~$\|G(\pi^{t*})\|_2$ (as defined in Lemma~\ref{lem:approxsmooth22}) can be bounded as follows: 
\begin{equation}
\|G(\pi^{t*})\|_2 \leq 2 \|\nabla \phi_{\frac{\lambda}{2}}(\pi^{t*})\|_2 = \frac{4}{\lambda} \|y^{(t*+1)} - \pi^{(t*)}\|_2\,,
\end{equation}
where the last equality and inequality follow respectively from equation (1.3) p.~2 and the last inequality p.~10 in \cite{DD18}. 

Hence, we need to choose $\eta, T$ so that
\[
\sqrt{\frac{2}{\eta T} + \frac{\eta}{\epsilon} \frac{48n A_{\max}^2}{\lambda (1-\gamma)^4}} \leq \frac{\epsilon (1-\gamma)^4}{2DS(nA_{max})^{3/2}}.
\]
We conclude that $\eta$ can be chosen to be $\frac{\epsilon^3(1-\gamma)^{15}}{768 n^5 \gamma D^2S^2 A^6_{\max} }$
and $T$ to be $\frac{12288 \gamma n^8 A_{\max}^9 D^4 S^4 }{\epsilon^5 (1-\gamma)^{23}}.$
\end{proof}

\begin{remark}\label{rem:markov}
Using Markov's inequality, it is immediate to show that the statement of \Cref{thm:mainformalb} holds with high probability. Namely, if we set the number of iterations to be $T=\frac{12288 \gamma n^8 A_{\max}^9 D^4 S^4 }{\epsilon^5 (1-\gamma)^{23} \delta^4}$ and the learning rate (step-size) to be $\eta = \frac{\epsilon^3(1-\gamma)^{15} \delta^2}{768 n^5 \gamma D^2S^2 A^6_{\max} }$, where $\delta\in(0,1)$, then with probability $1-\delta$ there exists a $t\in \{1,\dots,T\}$ such that $\pi^{(t)}$ is an $\epsilon$-approximate Nash policy. However, this is a weaker than desired statement, since, optimally, the running time should be logarithmic in $1/\delta$, (and not polynomial as above) \cite{Nem09}. Proving such a statement though requires bounds on the higher moments of the gradient estimator (to apply martingale arguments and concentration inequalities with exponential bounds) which we could not derive using current techniques (cf. Lemma 2 in \cite{Das20}). Such a bound would be possible if we sample multiple trajectories and take the average (per iteration).
\end{remark}

\begin{remark}[Weighted and ordinal MPGs]\label{rem:oMPGs} We conclude this section by giving a remark on Weighted and Ordinal MPGs (cf. Definition in \ref{rem:owMPG}). It is rather straightforward to see that our results carry over for weighted MPGs. The only difference in the running time of (\ref{eq:pga}) is to account for the weights (which are just multiplicative constants). 

In contrast, the extension to ordinal MPGs is not immediate and the reason is that we cannot prove any bound on the smoothness of $\Phi$ in that case (i.e., we cannot generalize \Cref{claim:smoothness}). Therefore, we cannot have rates of convergence of policy gradient. Nevertheless, it is quite straightforward that \eqref{eq:pga} converges asymptotically to critical points (in bounded domains) for differentiable functions. Therefore as long as $\Phi$ is differentiable, it is guaranteed that asymptotically \eqref{eq:pga} will converge to a critical point of $\Phi$. By \Cref{lem:stationary}, this point will be a Nash policy.
\end{remark}

%% file: experiments.tex
\section{Experiments: Congestion Games}\label{sec:experiments}

We next study the performance of the policy gradient algorithm in a general class of MDPs that are congestion games at every state, \cite{Mon96,Rou15}.

\begin{figure}[!tb]
\centering
\raisebox{1em}{\hspace{0cm}
\begin{tikzpicture}[scale=0.33,every node/.style={minimum size=0cm}]
 \begin{scope}[yshift=-63,every node/.append style={yslant=0.5,xslant=-1},yslant=0.5,xslant=-1]

       \draw[step=25mm, black] (0,0) grid (5,5);
       \draw[black,very thick] (0,0) rectangle (5,5);
       \fill[red] (2,3.5) circle (0.1); 
       \fill[blue] (1.6,1) circle (0.1); 
       \fill[red] (1.5,3.2) circle (0.1); 
       \fill[red] (0.3,4.3) circle (0.1); 
       \fill[red] (0.5,3.4) circle (0.1); 
       \fill[blue] (3.8,0.7) circle (0.1); 
       \fill[red] (1.4,4.2) circle (0.1); 
       \fill[blue] (3.2,3.7) circle (0.1); 
\end{scope}
\begin{scope}[yshift=30,every node/.append style={yslant=0.5,xslant=-1},yslant=0.5,xslant=-1]
       \fill[white,fill opacity=0.7] (0.2,0) rectangle (5.2,5);
       \draw[black,very thick] (0.2,0) rectangle (5.2,5);
       \draw[step=25mm, black, xshift=.2cm] (0,0) grid (5,5);
       \fill[blue] (1.6,1.8) circle (0.1); 
       \fill[blue] (2.3,1) circle (0.1); 
       \fill[blue] (4,1) circle (0.1);
       \fill[blue] (4.2,3.4) circle (0.1); 
       \fill[blue] (4.9,0.8) circle (0.1); 
       \fill[blue] (4.4,4) circle (0.1);        \fill[blue] (2.4,4) circle (0.1);       \fill[blue] (1.8,3.5) circle (0.1);
   \end{scope}
\node at (5, 0.2)  (a) {};
\node at (-5, 0.2)  (b) {};
\node at (-4.8, 3.7) (c){};
\node at (5.2, 3.7)  (d) {};

\draw[-latex] (d) edge[bend left] node [midway,left,blue,very thick, xshift=0.72cm] {\scriptsize $\le  N/4$} (a);
\draw[-latex] (b) edge[bend left] node [midway,right,red,thick,xshift=-0.45cm] {\scriptsize $>N/2$} (c);

\draw[-latex,gray!70!black](-3,5.8)node[left]{\scriptsize 4 facilities}
       to[out=-40,in=120] (0.4,2.2);
\draw[-latex,gray!70!black](-3,5.8)node[left]{}
       to[out=-40,in=120] (1,5);
\draw[-latex,gray!70!black](-3,5.8)node[left]{}
       to[out=-40,in=120] (3.2,3.7);
\draw[-latex,gray!70!black](-3,5.8)node[left]{ }
       to[out=-40,in=120] (-2.1,3.5);
   \draw[thick,red](4.9,5.8) node {\scriptsize distancing state};
   \draw[thick,blue](4.3,-1.3) node {\scriptsize safe state};
\end{tikzpicture}}\hspace{0.1cm}
\includegraphics[width=0.31\textwidth]{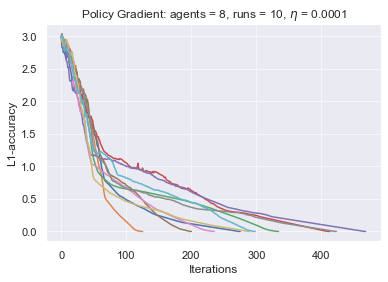}\hspace{0.1cm}
\includegraphics[width=0.31\textwidth]{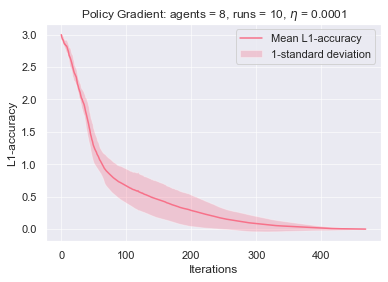}
\\[0.2cm]
\includegraphics[width=0.33\textwidth]{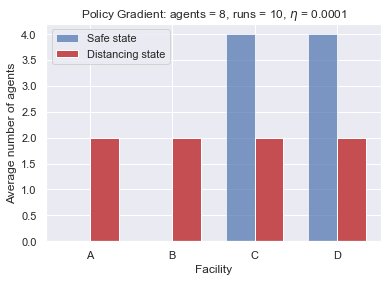}\hspace{0.1cm}
\raisebox{0.2em}{\includegraphics[width=0.31\textwidth]{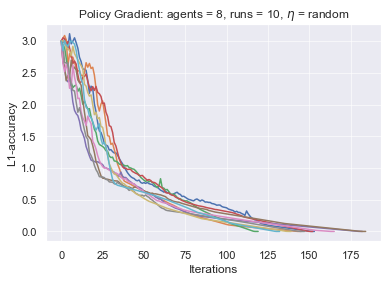}\hspace{0.1cm}
\includegraphics[width=0.31\textwidth]{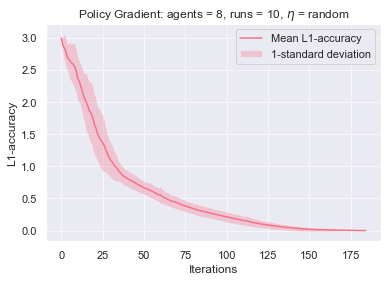}}
\vspace*{-0.2cm}
\caption{Upper left panel: An illustration of the MDP that is used in the experiments with $S=2$ states, $4$ facilities and $N=8$ agents (description in text). Lower left panel: The distribution of agents at the equilibrium that is reached by the policy gradient algorithm (common for all runs). Right column: Trajectories of the L1-accuracy (average difference between current policy and Nash policy) over the 10 runs for both equal (upper panels) and different learning rates among agents (lower panels).}
\label{fig:congestion}
\end{figure}
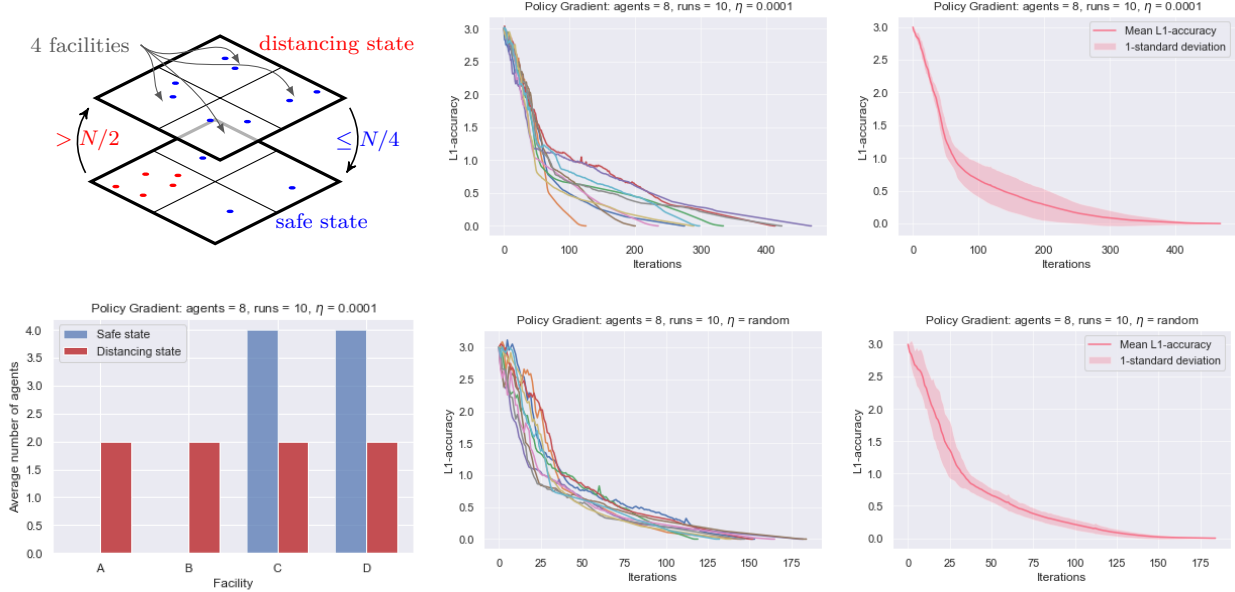

\paragraph{Experimental setup:} We consider a MDP in which every state is a congestion game (cf. \cite{Bist20}). In the current experiment, there are $N=8$ agents, $A_i=4$ facilities (resources or locations) that the agents can select from and $S=2$ states: a \emph{safe} state and a \emph{distancing} state. In both states, all agents prefer to be in the same facility with as many other agents as possible \emph{(follow the crowd)} \cite{Has03}. In particular, the reward of each agent for being at facility $k$ is equal to a predefined positive weight $w_k^{\text{safe}}$ times the number of agents at $k=A,B,C,D$. The weights satisfy $w_A^{\text{safe}}<w_B^{\text{safe}}<w_C^{\text{safe}}<w_D^{\text{safe}}$, i.e., facility $D$ is the most preferable by all agents. However, if more than $4=N/2$ agents find themselves in the same facility, then the game transitions to the distancing state. At the distancing state, the reward structure is the same for all agents, but reward of each agent is reduced by a constant amount, $c$, where $c>0$ is a (considerably large) constant. (We also treat the case in which $c$ is different for each facility in \Cref{app:experiments}). To return to the safe state, the agents need to achieve maximum distribution over the facilities, i.e., no more than $2=N/4$ agents may be in the same facility. We consider deterministic transitions, however, the results are quantitatively equivalent also when these transitions occur with some probability (see \Cref{app:experiments}). The MDP is illustrated in the upper left panel of \Cref{fig:congestion}.

\paragraph{Paremeters:} We perform episodic updates with $T=20$ steps. At each iteration, we estimate the Q-functions, the value function, the discounted visitation distributions and, hence, the policy gradients using the average of mini-batches of size $20$. We use $\gamma=0.99$. For the presented plots, we use a common learning rate $\eta=0.0001$ (upper panels) or randomly generated learning rates (different for each agent) in $[.00005, .0005]$ (lower panels) in \Cref{fig:congestion}. Note that these learning rates are (several orders of magnitude) larger than the theoretical guarantee, $\eta=\frac{(1-\gamma)^3}{2\gamma A_{\max} n} \approx 1e-08$, of \Cref{thm:mainformal}. Experiments (not presented here) with even larger learning rates (e.g., $\eta=0.001$) did not lead (consistently) to convergence.

\paragraph{Results:}
The lower left panel of \Cref{fig:congestion} shows that the agents learn the expected Nash profile in both states in all runs (this is common for both the fixed and the random learning rates). At the safe state, the agents distribute themselves equally among the two most preferable facilities (C and D). This leads to a maximum utility and avoids a transition to the distancing (bad) state. At the distancing state, the agents learn the unique distribution (2 agents per facility) that leads them back to the safe state. Importantly, this (Nash) policy profile to which policy gradient converges to is \emph{deterministic} in line with \Cref{thm:mainformal}. The panels in the middle and right columns depict the L1-accuracy in the policy space at each iteration which is defined as the average distance between the current policy and the final policy of all $8$ agents, i.e., $\text{L1-accuracy}=\frac1N\sum_{i\in \N}|\pi_i-\pi_i^{\text{final}}|=\frac1N\sum_{i\in \N}\sum_{s}\sum_{a}|\pi_i(a\mid s)-\pi_i^{\text{final}}(a\mid s)|.$ The results are qualitatively equivalent in both cases (common and non-common learning rates). However, due to the larger step-sizes used by some agents, the algorithm becomes more responsive and exhibits faster convergence in the non-common case. 

%% file: conclusions.tex
\section{Further Discussion and Conclusions}
\label{sec:conclusions}

In this paper, we have presented a number of positive results (both structural and algorithmic) about the performance of independent policy gradient ascent in Markov potential games. Specifically, deterministic Nash policies always exist and independent policy gradient is guaranteed to converge (polynomially fast in the approximation error) to a Nash policy profile even in the case of finite samples. Given these positive results, a number of interesting open questions emerge.

\paragraph{Open questions.}

\textbf{Price of (Markov) Anarchy.} Price of Anarchy (PoA)~\cite{poa} is a classic notion in normal form games that captures the inefficiency due to the lack of a central authority that would coordinate all agents to implement the social optimum outcome. Formally, it is defined as the ratio of the social cost of the worst Nash equilibrium divided by the cost of the social optimum. PoA has been studied extensively in many classes of games including several classes of potential games for which tight PoA bounds exist (e.g. congestion games~\cite{roughgarden2002bad}). It would be interesting to explore to what extent this type of analysis can be generalized to Markov Potential Games as well as more general classes of Markov Games.

\textbf{Stability of Deterministic Policies.} When it comes to online learning in normal form potential games, it is sometimes possible to prove that the dynamics do not converge to an arbitrary Nash equilibrium, but, in fact, that most initial conditions converge to a deterministic (sometimes referred to also as pure) Nash equilibrium~\cite{Kleinberg09multiplicativeupdates,panageas2018multiplicative,cohen2017learning,ITCS15MPP}. To produce such equilibrium selection results, standard Lyapunov arguments do not suffice and one needs to apply more advanced techniques such as the Center-Stable-Manifold theorem~\cite{lee2019first} which would be a fascinating direction for future work in MPGs. 

\textbf{Other Algorithmic Approaches: Softmax Parametrization \& Natural Policy Gradient.}
In \cite{Aga20}, the authors show asymptotic convergence to the global optimum to single-agent MDP
in the tabular setting with exact gradients
for the softmax policy parameterization. Moreover, 
 polynomial convergence rate is shown when additional KL-based entropy regularizer is used, as well as 
dimension-free convergence to optimum when Natural Policy Gradient is applied. Extending such algorithmic techniques to the case of multi-agent MPGs is a natural direction for future work. 

\textbf{Global Convergence in other Multi-Agent Markov Games.} Recently, there has been intense interest in understanding convergence to Nash policies for different classes of learning dynamics in Markov zero-sum games \cite{Das20,mdplastiterate,newZS}. Our approach moves in orthogonal direction focusing on MPGs and establishing strong convergence results in these games. A natural open question  is whether and under what conditions can we prove strong convergence guarantees in more general classes of Markov games, possibly by combining tools and techniques from both  lines of work.

\textbf{Regularities beyond Equilibration in Multi-Agent Markov Games.} Given the complexities of such multi-agent settings, it is highly unlikely to expect practical algorithms which can always guarantee convergence to equilibrium. This is already the case even for the more restricted setting of normal-form games~\cite{vlatakis2020no,andrade2021learning}. Nevertheless, strong guarantees can be shown via, e.g., existence of cyclic/recurrent orbits, invariant functions~\cite{mertikopoulos2018cycles} or strong social welfare guarantees~\cite{Syrgkanis:2015:FCR:2969442.2969573}. Whether such results can be extended to Multi-Agent Markov Games is a stimulating direction for future work.

%% file: appendix.tex
\section{Additional Notation and Definitions: 
\texorpdfstring{\Cref{sec:notation}}{Section \ref{sec:notation}}
}\label{app:notation}
We first provide some additional notation and definitions that will be used in the proofs.

\paragraph{Q-value and Advantage Functions.} Recall from the main part that the value function, $V_s^i:\Pi\to\mathbb R$, gives the expected reward of agent $i\in \N$ when $s_0=s$ and the agents draw their actions, $\a_t=(a_{i,t}, \a_{-i,t})$, at time $t\ge0$ from policies $\pi=(\pi_i,\pi_{-i})$ and is defined as
\[V^i_s(\pi) := \ex_{\tau\sim\pi} \left[\sum_{t=0}^\infty \gamma^t r_{i,t}\mid s_0=s\right].\]
Similarly, we will write $V^i_\rho(\pi):=\ex_{s_o\sim \rho}[V^i_s(\pi)]$ to denote the expected value of agent $i\in \N$ under the initial state distribution $\rho$. \par
For any state $s\in \S$, the Q-value function $Q_s^i:\P\times \A \to \mathbb{R}$ and the advantage function $A_s^i:\P\times \A \to \mathbb R$ of agent $i\in \N$ are defined as 
\begin{align}\label{eq:qvalue}
    Q^i_{s}(\pi,\a)&:=\ex_{\tau\sim\pi}\left[\sum_{t=0}^\infty \gamma^t r_{i,t}\mid s_0=s,\a_0=\a\right], \text{ and} \\[0.15cm]
    A^i_{s}(\pi,\a)&:=Q^i_{s}(\pi,\a)-V^i_s(\pi).
\end{align}

\paragraph{Discounted State Distribution.} It will be useful to define the discounted state visitation distribution $d_{s_0}^\pi(s)$ for $s\in \S$ that is induced by a (joint) policy $\pi$ as
\begin{equation}\label{eq:visitation}
    d_{s_0}^\pi(s):=(1-\gamma)\sum_{t=0}^{\infty}\gamma^t \Pr^\pi(s_t=s\mid s_0), \quad \text{ for all } s\in \S.
\end{equation}
As for the value function, we will also write $d_\rho^\pi(s)=\ex_{s_0\sim\rho}[d_{s_0}^\pi(s)]$ to denote the discounted state visitation distribution when the initial state distribution is $\rho$.

\section{Omitted Materials:
\texorpdfstring{\Cref{sec:characterization}}{Section \ref{sec:characterization}}
}\label{app:omitted}
\begin{proposition}[Separability of Value Functions and Equality of Derivatives]\label{prop:separability}
Let $\G=(\S, \N, \A = \{\A_i\}_{i \in \N}, P, R, \rho)$ be a Markov Potential Game (MPG) with potential $\Phi_s$, for $s \in S$. Then, for the value function $V_s^i, s\in \S$ of each agent $i\in \N$, the following hold
\begin{itemize}[leftmargin=0.7cm]
\item[P1.] Separability of Value Functions: there exists a function $U^i_s:\Delta(\A_{-i})^S\to \mathbb R$ such that 
for each joint policy profile $\pi=(\pi_i,\pi_{-i})\in \Delta(\A)^S$, we have $V_s^i(\pi) = \Phi_s(\pi) + U_s^i(\pi_{-i})$.
\item[P2.] Equality of Derivatives: the partial derivatives of agent $i$'s value function $V^i_s$ coincide with the partial derivatives of the potential $\Phi_s$ that correspond to agent $i$'s parameters, i.e.,
\[\partial_{x_{i,s,a}} V_s^i(\pi)=\partial_{x_{i,s,a}} \Phi_s(\pi), \quad \text{ for all } i\in \N \text{and all } s\in \S.\]
\end{itemize}
\end{proposition}

\begin{proof}
To obtain P1, consider any 3 arbitrary policies for agent $i$, notated by $\pi_i, \pi_i', \pi_i''\in \Delta(\A_i)^S$. Then, by the definition of MPGs, we have that 
\begin{align*}
    \Phi_s(\pi_i, \pi_{-i}) - \Phi_s(\pi_i', \pi_{-i}) &= V^i_s(\pi_i, \pi_{-i}) - V^i_s(\pi_i', \pi_{-i}), \\
    \Phi_s(\pi_i, \pi_{-i}) - \Phi_s(\pi_i'', \pi_{-i}) &= V^i_s(\pi_i, \pi_{-i}) - V^i_s(\pi_i'', \pi_{-i}). 
\end{align*}
for every starting state $s\in \S$. This implies that we can write $V^i_s(\pi_i,\pi_{-i})$ as both 
\begin{align*}
    V^i_s(\pi_i, \pi_{-i}) &= \Phi_s(\pi_i, \pi_{-i}) - \Phi_s(\pi_i', \pi_{-i}) + V^i_s(\pi_i', \pi_{-i}) \\
    V^i_s(\pi_i, \pi_{-i}) &= \Phi_s(\pi_i, \pi_{-i}) - \Phi_s(\pi_i'', \pi_{-i}) + V^i_s(\pi_i'', \pi_{-i})
\end{align*}

Thus, we have that $- \Phi_s(\pi_i', \pi_{-i}) + V^i_s(\pi_i', \pi_{-i}) = - \Phi_s(\pi_i'', \pi_{-i}) + V^i_s(\pi_i'', \pi_{-i})$ for any arbitrary pair of policies $\pi_i'$ and $\pi_i''$ for agent $i$, implying that agent $i$'s policy has no impact on these terms. Accordingly, we can express them as 
\[U^i_s(\pi_{-i}):=- \Phi_s(\pi_i', \pi_{-i}) + V^i_s(\pi_i', \pi_{-i}) = - \Phi_s(\pi_i'', \pi_{-i}) + V^i_s(\pi_i'', \pi_{-i}),\] where $U^i_s(\pi_{-i})$ is a function that does not depend on the policy of agent $i$. Thus, we can express the utility function of any agent $i$ in a MPG as
\begin{align*}
    V^i_s(\pi) = \Phi_s(\pi) + U^{i}_s(\pi_{-i}).
\end{align*}
as claimed. To obtain P2, we use P1 for a vector $x_i$ parameterizing $\pi_i$, and obtain that 
\begin{align*}
    \partial_{x_{i,a}} V^i(\pi) = \partial_{x_{i,a}} \Phi(\pi) + 0
\end{align*}
for any coordinate $x_{i,a}$ with $a\in A_i$ of $x_i$, from which we can see that our claim is true.
\end{proof}
Note that P1 serves as a characterization of MPGs. Namely, a multi-agent MDP is a MPG if and only if the value function of each agent $i\in \N$ can be decomposed in a term that is common for all players (potential function) and in a term that may be different for each agent $i\in \N$ but which depends only on the actions of all agents other than $i$. This property carries over from normal form (single state) potential games. Also note that both properties, P1 and P2, hold for any (differentiable for P2) policy parameterization and not only for the direct one that we use here. 

\begin{proof}[Proof of \Cref{thm:determininstic}]
Let $\Phi$ be the potential function of $\G$. Since the space $\Delta(\A)^S=\Delta(\mathcal{A}_1) ^S \times ... \times \Delta(\mathcal{A}_n)^S$ is compact and $\Phi$ is continuous, $\Phi$ has a global maximum $\Phi_{\max}$. Let $(\pi_1^*,...,\pi_n^*)$ denote a global maximizer, i.e., a joint policy profile at which $\Phi_{\max}$ is attained. By the Definitions of MPGs and Nash policies, this implies, in particular, that $(\pi_1^*,...,\pi_n^*)$ is a Nash policy, since
\begin{equation}\label{eq:tag}
0<\Phi_s(\pi_i^*,\pi_{-i}^*)-\Phi_s(\pi_i,\pi_{-i}^*)=V_s^i(\pi_i^*,\pi_{-i}^*)-V_s^i(\pi_i,\pi_{-i}^*),\tag{$\ast$}
\end{equation}
for all $i\in \N, s\in \S$ and all policies $\pi_i\in \Delta(\A_i)^S$. If $\pi_1^*,...,\pi_n^*$ are all deterministic we are done. So, we may assume that there exists an $i\in \N$ so that $\pi_i^*$ is randomized and consider the MDP $\G'$ in which the policy of all agents other than $i$ has been fixed to $\pi_{-i}^*$. $\G'$ is a single agent MDP with the same states as $\G$, the same actions and rewards for agent $i$ and transition probabilities that are determined by the joint distribution of the environment and the joint policy of all agents other than $i$. As a single agent MDP, this setting has a deterministic optimal policy, say $\tilde{\pi}_i$, for agent $i$. Thus, it holds that
\[V_s^i(\tilde\pi_i,\pi_{-i}^*)\le V_s^i(\pi_i^*,\pi_{-i}^*)\le V_s^i(\tilde\pi_i,\pi_{-i}^*),\]
where the first inequality follows from the fact that $\pi^*$ is a Nash policy and the second from the optimality of $\tilde{\pi_i}$. It follows that
\[V_s^i(\tilde\pi_i,\pi_{-i}^*)= V_s^i(\pi_i^*,\pi_{-i}^*),\]
i.e., the payoff of agent $i$ at $(\tilde{\pi}_i , \pi_{-i}^*)$ is the same as in $(\pi_i^*,\pi_{-i}^*)$. Hence, by the definition of the potential function, we have that 
\begin{align*}
    0&=V_s^i(\tilde\pi_i,\pi_{-i}^*)- V_s^i(\pi_i^*,\pi_{-i}^*)=\Phi_s(\tilde\pi_i,\pi_{-i}^*)- \Phi_s(\pi_i^*,\pi_{-i}^*),
\end{align*}
which implies that 
\[\Phi_s(\tilde\pi_i,\pi_{-i}^*)= \Phi_s(\pi_i^*,\pi_{-i}^*)=\Phi_{\max}.\]
Thus, $(\tilde{\pi}_i , \pi_{-i}^*)$ is also a global maximizer of $\Phi$ which implies that $(\tilde{\pi}_i , \pi_{-i}^*)$ is a Nash policy by the same reasoning as in equation \eqref{eq:tag}. Note that the value
of all players other than $i$ may not be the same at the joint policy profile $(\tilde{\pi}_i , \pi_{-i}^*)$
as it is in $(\pi^*_i , \pi_{-i}^*)$. However, what we need for our purpose is that this step reduces the number of randomized policies by one and that it retains the value of the potential function invariant at its global maximum (which ensures that the ensuing policy profile is also a Nash policy). By iterating this process until $\pi_j^*$ becomes deterministic for all agents $j\in\N$, we obtain the claim.
\end{proof}

\begin{proof}[Proof of \Cref{prop:conditions}]
The proof is constructive and proceeds by finding the potential function, $\Phi_s, s\in \S$ in both cases, C1-C2. Since the individual rewards of the agents at each state $s\in \S$ are captured by a potential function $\phi_s$, then for the reward, $R_i(s,\a)$ of each agent $i$ at the action profile $\a$, it holds that 
\begin{equation}\label{eq:separability}
    R_i(s,\a)=\phi_s(\a)+u^i_s(\a_{-i}),
\end{equation}
where $u^i_s:\Delta(\A_{-i})\to\mathbb R$ is a function that does not depend on the actions of agent $i$ in any way. Thus, we may write the value function of each agent $i\in \N$ as 
\begin{align}\label{eq:star}
    V_s^i(\pi)&=\ex_{\tau\sim\pi}\lt\sum_{t=0}^{\infty}\gamma^t R_i(s_t,\a_t) \mid s_0=s \rt\nonumber\\&=\ex_{\tau\sim\pi}\lt\sum_{t=0}^{\infty}\gamma^t \< \phi_{s_t}(\a_t)+u^i_{s_t}(\a_{-i,t})\>\mid s_0=s \rt\nonumber\\
    &=\ex_{\tau\sim\pi}\lt\sum_{t=0}^{\infty}\gamma^t \phi_{s_t}(\a_t)\mid s_0=s \rt+\ex_{\tau\sim\pi}\lt\sum_{t=0}^{\infty}\gamma^t u^i_{s_t}(\a_{-i,t})\mid s_0=s \rt\tag{$\star$}
\end{align}
where $\tau\sim \pi$ is the random trajectory generated by policy $\pi$. To show that $\G$ is a MPG, it suffices to show that the value function of each agent $i\in \N$ can be decomposed in a term that is common for all agents (and which may depend on the actions of agent $i\in \N$) and in a term that does not depend (in any way) in the actions of agent $i$ (dummy term), cf \Cref{prop:separability}. The first term in expression \eqref{eq:star}, i.e., $\ex_{\tau\sim\pi}\lt\sum_{t=0}^{\infty}\gamma^t \phi_{s_t}(\a_t)\mid s_0=s \rt$, depends on the actions of all players and is common for all agents $i\in \N$ (and is thus, a good candidate for the potential function). The second term in expression \eqref{eq:star}, i.e., $\ex_{\tau\sim\pi}\lt\sum_{t=0}^{\infty}\gamma^t u_{s_t}^i(\a_{-i,t})\mid s_0=s \rt$, does not depend on player $i$ via the payoffs $u_{s_t}^i(\a_{-i,t})$, but, in general, it does depend on player $i$ via the transitions, $\tau\sim\pi$. The two cases in the statement of \Cref{prop:conditions} ensure precisely that this term is independent of the policy of agent $i$, in which case it is a dummy term for agent $i$ or that is also common for all players, in which case it can be included in the potential function. Specifically, we have that 
\begin{itemize}[leftmargin=0.7cm]
\item[C1.] If the transitions do not depend on the action of the players, we have that $\tau\sim P$, where $P$ is an exogenously given distribution function (state-wise). In this case, we have that 
\[\Phi_s(\pi):=\ex_{\tau\sim\pi}\lt\sum_{t=0}^{\infty}\gamma^t \phi_{s_t}(\a_t)\mid s_0=s \rt\]
is a potential function and $U_s^i(\pi_{-i}):=\ex_{\tau\sim\pi}\lt\sum_{t=0}^{\infty}\gamma^t u_{s_t}^i(\a_{-i,t})\mid s_0=s \rt$ is a dummy term that does not depend (in any way) on the policy of agent $i\in \N$. 
\item [C2.] 
Under the assumptions of condition C2, we will show that again  \[\Phi_{s'}(\pi):=\ex_{\tau\sim\pi}\lt\sum_{t=0}^{\infty}\gamma^t \phi_{s_t}(\a)\mid s_0=s'\rt\]
is a potential function for $G$ and that the same decomposition as in condition C1 of the value function of agent $i$ in a common and a dummy term applies. To see this, let $\pi_i,\pi_i'\in \Pi_i$ be two policies of agent $i$ and let $\pi=(\pi_i,\pi_{-i}),\pi'=(\pi'_i,\pi_{-i})$ where $\pi_{-i}$ is the fixed policy of all agents other than $i$. Then, using \eqref{eq:star}, we have that
\begin{align*}
    V_{s'}^i(\pi)-V_{s'}^i(\pi')=&\;\Phi_{s'}(\pi)-\Phi_{s'}(\pi')\\& +\ex_{\tau\sim\pi}\lt\sum_{t=0}^{\infty}\gamma^t u^i_{s_t}(\a_{-i,t})\mid s_0=s'\rt-\ex_{\tau\sim\pi'}\lt\sum_{t=0}^{\infty}\gamma^t u^i_{s_t}(\a_{-i,t})\mid s_0=s'\rt.
\end{align*}
Using the intermediate value theorem, there exists a policy $\xi_i$ which is a convex combination of $\pi_i,\pi_i'$ such that
\begin{align*}
&\ex_{\tau\sim\pi}\lt\sum_{t=0}^{\infty}\gamma^t u^i_{s_t}(\a_{-i,t})\mid s_0=s'\rt-\ex_{\tau\sim\pi'}\lt\sum_{t=0}^{\infty}\gamma^t u^i_{s_t}(\a_{-i,t})\mid s_0=s'\rt\\&=(\pi_i-\pi_i')^{\top}\nabla_{\pi_i}\ex_{\tau\sim(\xi_i,\pi_{-i})}\lt\sum_{t=0}^{\infty}\gamma^t u^i_{s_t}(\a_{-i,t})\mid s_0=s'\rt
\end{align*}
Since $\pi_i,\pi_i'$ correspond to probability distributions at every state $s\in\S$, their difference is equal to $0$ (state-wise). In turn, the displayed condition in C2 implies that \[\nabla_{\pi_i}\ex_{\tau\sim\xi}\lt\sum_{t=0}^{\infty}\gamma^t u^i_{s_t}(\a_{-i,t})\mid s_0=s'\rt=(c_s\mathbf{1})_{s\in\S},\]
which is enough to ensure that the dot product in the previous equation is equal to $0$ and the claim follows.
\end{itemize}
Summing up, in both cases, C1-C2, $\G$ is an MPG as claimed. 
\end{proof}

Note that the proof of the (trivial) case in which the instantaneous rewards of all agents $i\in \N$ are equal at each state $s\in \S$ (cf. \Cref{rem:conditions}) is similar. In this case, it is immediate to see that the instantaneous rewards are precisely given by the potential function at that state, i.e., it holds that $R_i(s,\a)=\phi_s(\a)$ for all $i\in \N$ and all $s\in\S$. In this case, it holds that $u_s^i(\a_{-i})\equiv 0$ for all $i\in \N$ and all $s\in \S$ and hence, 
\[\Phi_s(\pi):=\ex_{\tau\sim\pi}\lt\sum_{t=0}^{\infty}\gamma^t \phi_{s_t}(\a)\mid s_0=s \rt\]
is a potential function for $G$, and the dummy terms are all equal to $0$, i.e., $U_s^i(\pi_{-i})\equiv 0$. 

\subsection{Examples}\label{app:examples}

\addtocounter{example}{-2}

\begin{example}[Continued] At each state, $s\in\{0,1\}$, the agents' payoffs, $(R_s^1,R_s^2)$, form a potential game (at that state), and are given as follows
\begin{align*}
\text{State $0:$ \;\;}& (R_0^1,R_0^2)=\bordermatrix{~ & 0 & 1 \cr 0 & \phantom{-}5,\phantom{-}2 & -1,-2\cr 1 & -5,-4 & \phantom{-}1,\phantom{-}4\cr}, \,\,\text{ with potential \;} \Phi_0=\begin{pmatrix}\phantom{-}4 & 0 \\-6 & 2 \end{pmatrix},\\[0.2cm]
\text{State $1:$ \;\;}& (R_1^1,R_1^2)=(0,0), \,\,\text{ with potential \;} \Phi_1=0.
\end{align*}

In this MDP, agents need only to select an action at state $s_0$. Thus, we will denote a policy, $\pi_1$, of agent $1$ by $\pi_1=(p,1-p)$ where $p\in[0,1]$ is the probability with which agent $A$ selects action $0$ at state $s_0$. Similarly, we will denote a policy, $\pi_2$, of agent $B$ by $\pi_2=(q,1-q)$ where $q\in[0,1]$ is the probability with which agent $B$ selects action $0$ at state $s_0$. Moreover, we will slightly abuse notation and write 
\[R_0^i(\pi)=R_0^i(\pi_1,\pi_2)=\pi_1^\top R_0^i\pi_2=[p,  1-p]R_0^i[q, 1-q]^\top.\]
We also assume that the horizon is infinite and there is a discount factor $\gamma\in[0,1)$. Accordingly, we can calculate the value functions $V_0^i(\pi_1,\pi_2)$ of agents $i=A,B$ starting from state $s_0$ as follows,
\begin{align*}
    V_0^i(\pi)&=R_0^i(\pi)+\gamma pqV_0^i(\pi)-\gamma(1-pq)\times 0
\end{align*}
which yields the solution
\begin{align*}
    V_0^i(\pi)&=\frac{R_0^i(\pi)}{1-\gamma pq}, \quad \text{for } i=A,B.
\end{align*}

Next, we use the Performance Difference Lemma (Lemma 3.2 by \cite{Aga20}) to determine the difference in the value between two different policies. We will do this for agent $1$ (the calculation is similar for agent $2$: we use here $1$ for agent 1 and 2 for agent $B$). For a policy $\pi=(\pi_1,\pi_2)$, we have at state $s_0$ that 
\begin{align*}
    \ex_{a\sim \pi_1(\cdot\mid s_0)} [A^1_0(\pi',\a)]&=p A^1_0(\pi',0,a_2)+(1-p)A^1(\pi',1,a_2)\\[0.2cm]
    &=p\lt R_0^1(0,\pi_2)+\gamma q V_0(\pi')-V_0(\pi') \rt+(1-p)\lt R_0^1(1,\pi_2)+0-V_0(\pi')\rt\\[0.2cm]
    &=pR_0^1(0,\pi_2)+(1-p)R_0^1(0,\pi_2)-(1-\gamma pq)V_0(\pi')\\[0.2cm]
    &=R_0^1(\pi)-(1-\gamma pq)V_0(\pi').
\end{align*}
At state $s_1$, there is only one available action for each agent which yields a payoff of $0$. Thus,
\begin{align*}
    \ex_{a\sim \pi_1(\cdot\mid s_1)} [A^1_1(\pi',\a)]&=0.
\end{align*}
Moreover, concerning the discounted visitation distribution, we have that  
\begin{align*}
    d^\pi_0(s_0)&=(1-\gamma)\sum_{t=0}^\infty\gamma^t\Pr^\pi(s_t=s_0\mid s_0)=(1-\gamma)\lt1+\gamma pq+(\gamma pq)^2+\dots \rt=\frac{1-\gamma}{1-\gamma pq},
\end{align*}
and $d^\pi_0(s_1)=1-d^\pi_0(s_0)=\frac{\gamma(1-pq)}{1-\gamma pq}$. Thus, using all the above, we have that
\begin{align*}
    V_0(\pi)-V_0(\pi')&=\frac{1}{1-\gamma}\lt \frac{1-\gamma}{1-\gamma pq} \cdot (R_0^1(\pi)-(1-\gamma pq)V_0(\pi'))+ \frac{\gamma(1-pq)}{1-\gamma pq}\cdot 0\rt\\
    &=\frac{R_0^1(\pi)}{1-\gamma pq}-V_0(\pi')=V_0(\pi)-V_0(\pi').
\end{align*}
which shows that our initial calculations conform with the outcome specified by the Performance Difference Lemma.\par
Finally, a direct calculation shows that $\Phi_s=\phi_s$ for $s=0,1$ is a valid potential function for which the MDP is an \emph{ordinal} MPG.
\end{example}

\begin{example}[Continued] At state $s_0$, we  consider the game with action sets  $A_1(s_0)=A_2(s_0)=\{ H,T\}$ and (instantaneous) payoffs 
\[R_1(s_0,a_1,a_2)=\bordermatrix{
~ & H & T \cr
H & \phantom{-}1 & -1\cr
T & -1 & \phantom{-}1 \cr}\quad  \text{and} \quad R_2(s_0,a_1,a_2)=\bordermatrix{
~ & H & T \cr
H & -1 & \phantom{-}1\cr
T & \phantom{-}1 & -1 \cr},\]
where $a_1$ denotes the action of agent $1$ and $a_2$ the action of agent $2$ (agent $1$ selects rows and agent $2$ selects columns in both matrices). This is a constant sum game (equivalent to zero-sum) and hence, it is not an (ordinal) potential game. Apart from the instantaneous rewards, agents' actions at $s_0$ induce a deterministic transition to a state in which the only available actions to the agents are precisely the actions that they chose at state $s_0$ and their instantaneous rewards at this state are the rewards of the other agent at $s_0$. In particular, there are four possible transitions to states $s_{ab}$ with $a,b \in \{H,T\}$, with action sets and instantaneous rewards given by \[A_1(s_{ab})=\{a\}, A_2(s_{ab})=\{b\}, \quad R_1(s_{ab},a,b)=R_2(s_0,b,a), \,\, R_2(s_{ab},a,b)=R_1(s_0,b,a),\] for agents 1 and 2, respectively. Note that the visitation probability of this states is equal to the visitation probability of state $s_0$. After visiting one of these states, the MDP transitions to state $s_1$ which is a potential game, with potential function given by \[\Phi_{1}=\bordermatrix{
~ & L & R \cr
L & 4 & 3 \cr
R & 3 & 0 \cr}\]
As mentioned above, the game in state $s_0$ does not admit a potential function. However, the joined rewards $RJ_1, RJ_2$ of agents $1$ and $2$ which result from selecting an action profile $(a,b)\in {H,T}^2$ at $s_0$ and then traversing both $s_0$ and the ensuing $s_{ab}$ (part included in the dotted rectangle in \Cref{fig:example}), do admit a potential function. The potential function in this case is the sum of agents' rewards and is given by
\[\Phi_{0ab}=\bordermatrix{
~ & H & T \cr
H & 1-1 & 1-1 \cr
T & 1-1 & 1-1 \cr}=\bordermatrix{
~ & H & T \cr
H & 0 & 0 \cr
T & 0 & 0 \cr}.\]
Let $\pi_1=(\xx_0,\xx_1)$ denote a policy of agent $1$. Here $\xx_0=(x_0,1-x_0)$, where $x_0\in[0,1]$ is the probability with which agent $1$ chooses action $H$ at state $s_0$. Similarly, $\xx_1=(x_1,1-x_1)$ where $x_1\in[0,1]$ is the probability with which agent $1$ chooses action $L$ at state $s_1$. At states $s_{ab}, a,b \in \{H,T\}^2$, agents only have one action to choose from, so this choice is eliminated from their policy representation. Similarly, we represent a policy of agent $2$ by $\pi_2=\<\yy_0,\yy_1\>$ with $y_0,y_1\in[0,1]$. Let also
\begin{equation}\label{eq:p01}
p_0:=p_0\<\pi_1,\pi_2\>:=\Pr(s_{t+1}=s_0\mid s_{t}=s_1, \pi_1,\pi_2),
\end{equation}
In the general case, $p_0$, i.e., the transition probability from $s_1$ to $s_0$, may depend on the actions of the agents or it may be completely exogenous (i.e., constant with respect to agents' actions). If we write 
\[
p_0(a_1,a_2):=\Pr(s_{t+1}=s_0\mid s_{t}=s_1, a_1,a_2), \quad \text{for } a_1,a_2\in\{L,R\},
\]
to denote the probability of transitioning from state $s_1$ to state $s_0$ given that the agents chose actions $a_1,a_2\in\{L,R\}$ at state $s_1$, then we can write $p_0$ as
\begin{align}\label{eq:p02}
p_0&=\ex_{(a_1,a_2)\sim(\pi_1,\pi_2)}[p_0(a_1,a_2)] =\sum_{(a_1,a_2)\in \{L,R\}^2}\Pr(a_1,a_2\mid \pi_1,\pi_2)\cdot p_0(a_1,a_2)\nonumber\\
&=x_1y_1p_0(L,L)+x_1(1-y_1)p_0(L,R)+(1-x_1)y_1p_0(R,L)+(1-x_1)(1-y_1)p_0(R,R).
\end{align}
Using this notation, we can now proceed to compute the value function of each state of the MDP in \Cref{fig:example}. Since the value of states $s_{a,b}, a,b \in {H,T}^2$ is equal to a constant reward plus the value of state $s_1$ (discounted by $\gamma$), it suffices to calculate the value for states $s_0$ and $s_1$. We have that
\begin{align*}
V_0^1\<\pi_1,\pi_2\>&=\xx_0R_0^1\yy_0+\gamma \<\xx_0\<R_0^2/\gamma\>\yy_0\>+\gamma^2V_1^1\<\pi_1,\pi_2\>\\
V_1^1\<\pi_1,\pi_2\>&=\xx_1R_1^1\yy_1+\gamma \lt p_0 V_0^1\<\pi_1,\pi_2\>+(1-p_0)V_1^1\<\pi_1,\pi_2\>\rt,
\intertext{which after some trivial calculations yield}
V_0^1\<\pi_1,\pi_2\>&=\xx_0\<R_0^1+R_0^1\>\yy_0+\gamma^2V_1^1\<\pi_1,\pi_2\>\\
V_1^1\<\pi_1,\pi_2\>&=\frac{1}{1-\gamma\<1-p_0\>}\lt \xx_1R_1^1\yy_1+\gamma p_0 V_0^1\<\pi_1,\pi_2\>\rt.
\end{align*}
This is a system of $2$ equations in the $2$ unknown quantities, $V_0^1\<\pi_1,\pi_2\>$ and $V_1^1\<\pi_1,\pi_2\>$. Solving for these two quantities, yields the unique solution
\begin{align*}
V_0^1\<\pi_1,\pi_2\>&=\frac{1}{1-\gamma\<1-p_0\>-\gamma^3p_0}\lt\<1-\gamma\<1-p_0\>\> \xx_0\<R_0^1+R_0^2\>\yy_0+\gamma^2 \xx_1R_1^1\yy_1\rt.\\
V_1^1\<\pi_1,\pi_2\>&=\frac{1}{1-\gamma\<1-p_0\>-\gamma^3p_0}\lt \gamma p_0 \xx_0\<R_0^1+R_0^2\>\yy_0+\xx_1R_1^1\yy_1\rt.
\end{align*}
In the case that $p_0$ is a constant with respect to $\pi_1,\pi_2$, then both value functions are of the form 
\[V_s^i\<\pi_1,\pi_2\>=c_1\<s\>\cdot \xx_0\<R_0^1+R_0^2\>\yy_0+c_2\<s\>\cdot \xx_1R_1^i\yy_1, \quad \text{for } s\in\{s_0,s_1\}, \,\,\text{and } i=\{1,2\},\]
where $c_1\<s\>,c_2\<s\>>0$ are appropriate constants that depend only the state $s\in\{s_0,s_1\}$ and on agents $1,2$. Since the game at $s_1$ is a potential game, with potential function given by a $2\times 2$ matrix $\Phi_1$, it is immediate to infer that  
\[V_s^i\<\pi'_1,\pi_2\>-V_s^i\<\pi_1,\pi_2\>=\Phi_s\<\pi'_1,\pi_2\>-\Phi_s\<\pi_1,\pi_2\>, \text{for } s\in\{s_0,s_1\},\]
with
\[\Phi_s\<\pi_1,\pi_2\>:=c_1\<s\>\xx_0\<R_0^1+R_0^2\>\yy_0+c_2\<s\>\cdot \xx_1\Phi_1\yy_1,\text{for } s\in\{s_0,s_1\}.\]
However, if $p_0$ depends on the actual policies of agents $1$ and $2$, cf. equation \eqref{eq:p01}, then it is not immediate to determine a potential (or even to decide whether a (exact) potential exists or not).

\begin{remark}
Several elements of \Cref{ex:noteverystate} have been selected in the sake of simplicity and are not necessary for the main takeaway, i.e., that there are MDP that are not potential at some states but which are MPGs. First, the transitions from $s_0$ to the states $s_{ab}$ need not be deterministic. To see this, let $q\in(0,1)$ and assume that if the agents select actions $H,T$ is $s_0$, then the process transitions with probability $q$ to a state $s_{HT}$ with rewards $(-1,1)/q\gamma$ and with probability $(1-q)$ to a state $s'_{HT}$ with rewards $(1,-1)/(1-q)\gamma$. The rest remains the same. Accordingly, the expected reward for agent $1$ after $(H,T)$ has been selected in $s_0$ is the same as in the current format. \par
Second, the construction with states $s_0$ and $s_{ab}, (a,b)\in {H,T}^2$ is not the only one that leads to such an example. Another very common instance occurs in the case of \emph{aliasing} between $s_0$ and states $s_{ab}$, i.e., when the agents cannot tell these states apart. The intuition which carries over from the currently presented example is that the roles of the agents are essentially reversed between the two states but the agents do not know (from the observable features) in which state they are. Thus, any valid policy, selects the same action in both states leading to the same situation as in the presented example.\par
Finally, if the horizon is finite, then the instantaneous rewards in states $s_{ab}$ still work if we eliminate the scaling factor (here $\gamma$). Thus, the construction works in both episodic and continuing settings.
\end{remark}
\end{example}

\section{Omitted Materials:
\texorpdfstring{\Cref{sec:convergence}}{Section \ref{sec:convergence}}
}\label{app:omitted_conv}



\begin{proof}[Proof of \Cref{claim:projection}]
Observe that for any set $\mathcal{X}\subseteq \mathbb R^n$, it holds that 
\[P_{\mathcal{X}}(y) = \argmin_{x \in\mathcal{X}}  \norm{x-y}_2^2.\]
Thus,
\begin{align*}
P_{\Delta(\mathcal{A})^S}(y) &= \argmin_{x \in \Delta(\mathcal{A})^S}  \norm{x-\pi'}_2^2= \argmin_{x_1 \in \Delta(\mathcal{A}_1)^S,...,x_n \in 
\Delta(\mathcal{A}_n)^S} \sum_{i=1}^n\norm{x_i - \pi'_i}_2^2\\
&= \sum_{i=1}^n \argmin_{x_i \in \Delta(\mathcal{A}_i)^S} \norm{x_i - \pi'_i}_2^2= (P_{\Delta(\A_1)^S} (\pi'_1) , \dots , P_{\Delta(\A_n)^S}(\pi'_n)).\qedhere
\end{align*}
\end{proof}

To prove \Cref{lem:gdom}, we will use a multi-agent version of the Performance Difference Lemma (cf. \cite{Aga20} for a single agent and \cite{Das20} for two agents).

\begin{lemma}[Multi-agent Performance Difference Lemma]\label{lem:perdif}
Consider an n-agent MDP $\G$ and fix an agent $i\in \N$. Then, for any policies $\pi=(\pi_i,\pi_{-i}), \pi'=(\pi'_i,\pi_{-i})\in \Pi$ and any distribution $\rho \in \Delta(S)$, it holds that 
\[V^i_\rho(\pi)-V^i_\rho(\pi')=\frac{1}{1-\gamma}\ex_{s\sim d_\rho^\pi}\ex_{a_i\sim \pi_i(\cdot\mid s)}\ex_{\a_{-i}\sim \pi_{-i}(\cdot\mid s)}\lt A_s^i(\pi',a_{i},\a_{-i}) \rt, \]
where $\a_{-i}\sim \pi_{-i}(\cdot \mid s)$ denotes the action profile of all agents other than $i$ that is drawn from the product distribution induced by their policies $\pi_{-i}=(\pi_j)_{j\neq i\in\N}\in \Pi_{-i}$.
\end{lemma}

\begin{proof}
For any initial state $s\in \S$ and joint policies $\pi=(\pi_i,\pi_{-i}),\pi'=(\pi'_i,\pi_{-i})\in \Pi$, it holds that 
\begin{align*}
    V^i_s(\pi)-V^i_s(\pi')&=\ex_{\tau\sim \pi}\lt \sum_{t=0}^{\infty}\gamma^t r_{i,t}\mid s_0=s\rt-V^i_s(\pi')\\&
    =\ex_{\tau\sim \pi}\lt \sum_{t=0}^{\infty}\gamma^t \<r_{i,t}-V_{s_t}(\pi')+V_{s_t}(\pi')\>\mid s_0=s\rt-V^i_s(\pi')\\&
    =\ex_{\tau\sim \pi}\lt \sum_{t=0}^{\infty}\gamma^t \<r_{i,t}-V_{s_t}(\pi')+\gamma V_{s_{t+1}}(\pi')\>\mid s_0=s\rt\\&
    =\ex_{\tau\sim \pi}\lt \sum_{t=0}^{\infty}\gamma^t \<r_{i,t}+\gamma\ex\lt V_{s_{t+1}}(\pi')\mid s_t,a_{i,t},\a_{-i,t}\rt-V_{s_t}(\pi')\>\mid s_0=s\rt\\&
    =\ex_{\tau\sim \pi}\lt \sum_{t=0}^{\infty}\gamma^t A_{s_t}^i(\pi',\a_t)\mid s_0=s\rt\\&
    =\frac{1}{1-\gamma}\ex_{s'\sim d_\rho^\pi}\ex_{a_i\sim \pi_i(\cdot\mid s')}\ex_{\a_{-i}\sim \pi_{-i}(\cdot\mid s')}\lt A_{s'}^i(\pi',\a) \rt.
\end{align*}
Taking expectation over the states $s\in \S$ with respect to the distribution $\rho \in \Delta(\S)$ yields the result.
\end{proof}

\begin{proof}[Proof of \Cref{lem:gdom}]
Fix an agent $i\in \N$ and let $\pi=(\pi_i,\pi_{-i}), \pi^*=(\pi_i^*,\pi_{-i})\in \Pi=\Delta(\A)^S$. By the definition of MPGs (cf. \Cref{def:potential}), it holds that 
\[V^i_{\rho}(\pi^*)-V^i_{\rho}(\pi)=\Phi_{\rho}(\pi^*)-\Phi_{\rho}(\pi).\]
Thus, using the multi-agent version of the Performance Difference Lemma (cf. \Cref{lem:perdif}), we have for any distribution $\mu\in \Delta(\S)$ that 
\begin{align*}
\Phi_{\rho}(\pi^*)-\Phi_{\rho}&(\pi) = V^i_{\rho}(\pi^*)-V^i_{\rho}(\pi)\\
&
=\frac{1}{1-\gamma}\ex_{s\sim d_\rho^{\pi^*}}\ex_{a_i\sim \pi^*_i(\cdot\mid s)}\ex_{\a_{-i}\sim \pi^*_{-i}(\cdot\mid s)}\lt A_s^i(\pi,a_{i,t},\a_{-i,t}) \rt\\&
\le \frac{1}{1-\gamma}\max_{\pi'_i\in\Pi_i}\left\{\sum_{s\in \S}d_\rho^{\pi^*}(s)\ex_{a_i\sim \pi'_i(\cdot\mid s)}\ex_{\a_{-i}\sim \pi^*_{-i}(\cdot\mid s)}\lt A_s^i(\pi,a_{i},\a_{-i}) \rt\right\},\\&
=\frac{1}{1-\gamma}\max_{\pi'_i\in \Pi_i}\left\{\sum_{s\in \S}\frac{d_\rho^{\pi^*}(s)}{d_\mu^{\pi}(s)}d_\mu^{\pi}(s)\ex_{a_i\sim \pi'_i(\cdot\mid s)}\ex_{\a_{-i}\sim \pi^*_{-i}(\cdot\mid s)}\lt A_s^i(\pi,a_{i},\a_{-i}) \rt\right\},\\&
\le\frac{1}{1-\gamma}\left\|\frac{d^{\pi^*}_\rho}{d^\pi_\mu}\right\|_{\infty}\max_{\pi'_i\in \Pi_i}\left\{\sum_{s\in \S}d_\mu^{\pi}(s)\ex_{a_i\sim \pi'_i(\cdot\mid s)}\ex_{\a_{-i}\sim \pi^*_{-i}(\cdot\mid s)}\lt A_s^i(\pi,a_{i},\a_{-i}) \rt\right\}.
\end{align*}
To proceed, observe that  
\[\ex_{a_i\sim \pi_i(\cdot\mid s)}\ex_{\a_{-i}\sim \pi^*_{-i}(\cdot\mid s)}\lt A_s^i(\pi,a_{i},\a_{-i}) \rt=0.\]
Thus, for any $\pi'_i\in \Pi_i$ and any state $s\in\S$, it holds that 
\begin{align*}
&\ex_{a_i\sim \pi'_i(\cdot\mid s)}\ex_{\a_{-i}\sim \pi^*_{-i}(\cdot\mid s)}\lt A_s^i(\pi,a_{i},\a_{-i}) \rt =\\=&
\sum_{a_i\in \A_i} \<\pi_i'(a_i\mid s)-\pi_i(a_i\mid s)\>\ex_{\a_{-i}\sim \pi^*_{-i}(\cdot\mid s)}\lt A_s^i(\pi,a_{i},\a_{-i}) \rt\\
=&\sum_{a_i\in \A_i} \<\pi_i'(a_i\mid s)-\pi_i(a_i\mid s)\>\ex_{\a_{-i}\sim \pi^*_{-i}(\cdot\mid s)}\lt Q_s^i(\pi,a_{i},\a_{-i}) \rt
\end{align*}
since $V^i_s(\pi)$ does not depend on $a_i$. Substituting back in the last inequality of the previous calculations, we obtain that 
\begin{align*}
&\Phi_{\rho}(\pi^*)-\Phi_{\rho}(\pi)\le\\
\le&\left\|\frac{d^{\pi^*}_\rho}{d^\pi_\mu}\right\|_{\infty}\max_{\pi'_i\in \Pi_i}\left\{\sum_{s,a_i}\frac{d_\mu^{\pi}(s)}{1-\gamma}\<\pi_i'(a_i\mid s)-\pi_i(a_i\mid s)\>\ex_{\a_{-i}\sim \pi^*_{-i}(\cdot\mid s)}\lt Q_s^i(\pi,a_{i},\a_{-i}) \rt\right\}\\=&
\left\|\frac{d^{\pi^*}_\rho}{d^\pi_\mu}\right\|_{\infty}\max_{\pi'_i\in \Pi_i}(\pi_i'-\pi_i)^{\top}\nabla_{\pi_i}V_\mu^i(\pi),
\end{align*}
where we used the policy gradient theorem (\cite{Aga20,Sut18}) under the assumption of direct policy parameterization (cf. equation \eqref{eq:parameter}). We can further upper bound the last expression by using that $d^{\pi}_\mu(s)\ge (1-\gamma)\mu(s)$ which follows immediately from the definition of the discounted visitation distribution $d^{\pi}_\mu(s)$ for any initial state distribution $\mu$. Finally, property P2 of Proposition \ref{prop:separability}, implies that $\nabla_{\pi_i} V^i_{\rho}(\pi)=\nabla_{\pi_i}\Phi_{\rho}(\pi)$ (making crucial use of the MPG structure). Putting these together, we have that 
\begin{align*}
\Phi_{\rho}(\pi^*)-\Phi_{\rho}(\pi)& 
\le \frac{1}{1-\gamma}\left \|\frac{d^{\pi^*}_\rho}{\mu}\right\|_{\infty}\max_{\pi'=(\pi'_i,\pi_{-i}^*)}(\pi'-\pi)^\top \nabla_{\pi_i}V^i_{\mu}(\pi)\\
&=\frac{1}{1-\gamma}\left \|\frac{d^{\pi^*}_\rho}{\mu}\right\|_{\infty}\max_{\pi'=(\pi'_i,\pi_{-i}^*)}(\pi'-\pi)^\top \nabla_{\pi_i}\Phi_{\mu}(\pi),
\end{align*}
as claimed.
\end{proof}

\begin{proof}[Proof of Lemma \ref{claim:smoothness}]
It suffices to show that the maximum eigenvalue in absolute value of the Hessian of $\Phi$ is at most 
$\frac{2n\gamma A_{\max}}{(1-\gamma)^3}$, i.e., that
\[\norm{\nabla^{2} \Phi_{\mu}}_2 \leq \frac{2n\gamma A_{\max}}{(1-\gamma)^3}\,.\] 
We first prove the following intermediate claim.

\begin{claim}\label{claim:algebra}
Consider the symmetric block matrix $C$ with $n\times n$ matrices so that $\norm{C_{ij}}_2 \leq L$. Then, it holds that $\norm{C}_2 \leq nL$, i.e., if all block submatrices have spectral norm at most $L$, then $C$ has spectral norm at most $nL$.
\end{claim}
\begin{proof}
We will prove the claim by induction on $n$. For $n=2$ we need to show that
\begin{equation*}
\norm{C}_2 :=
    \norm{\left(\begin{array}{cc}
    C_{11}     &  C_{12}\\
    C_{21}     &  C_{22}
    \end{array}\right)}_2 \leq 2L
\end{equation*}
if $\norm{C_{11}}_2, \norm{C_{12}}_2, \norm{C_{21}}_2,\norm{C_{22}}_2 \leq L.$
Define matrix $W$ to be
\begin{equation*}
    W:=2L \cdot I - C = \left(\begin{array}{cc}
    2L \cdot I - C_{11}     &  - C_{12}\\
    - C_{21}     &  2L \cdot I - C_{22}
    \end{array}\right),
\end{equation*}
where $I$ is the identity matrix (of appropriate size). If we show that $W$ is positive semi-definite, then it follows that $W$ has only non-negative eigenvalues, which, in turn, implies that the spectral norm of $C$ is at most $2L$. To see this, set
\begin{equation*}
    W_1:=\left(\begin{array}{cc}
    L \cdot I - C_{11}     &  0\\
    0     &  L \cdot I - C_{22}
    \end{array}\right),
    W_2:=\left(\begin{array}{cc}
    L \cdot I      &  -C_{12}\\
    -C_{21}     &  L \cdot I 
    \end{array}\right).
\end{equation*}
$W_1$ is positive semi-definite as a block diagonal matrix with diagonal blocks positive semi-definite matrices. Moreover, by Schur complement we get that $W_2$ is positive semi-definite as long as $L \cdot I$ is positive semi-definite and $L \cdot I - \frac{1}{L} \cdot C_{12} C_{21}$ is positive semidefinite. By assumption, we have that \[\frac{1}{L} \norm{ C_{12} C_{21}}_2 \leq \frac{1}{L}\norm{C_{12}}_2 \norm{C_{12}}_2 \leq L,\] which implies that $L \cdot I - \frac{1}{L} \cdot C_{12} C_{21}$ has non-negative eigenvalues. Thus, $W_2$ is positive semi-definite. We conclude that $W_{1}+W_{2}$ is positive semi-definite (sum of positive semi-definite matrices is positive semi-definite) and the claim follows. \par

For the induction step, suppose that the claim holds for an $n=k-1 \geq 2$. To establish that it also holds for $k$, we need to show that

\begin{equation*}
\norm{C}_2 :=
    \norm{\left(\begin{array}{cccc }
    C_{11}     &  C_{12} &\dots & C_{1k}\\
    C_{21}     &  C_{22} &\dots & C_{2k}\\
    \vdots & \vdots & \vdots & \vdots\\
    C_{k1} & C_{k2} &\dots & C_{kk}
    \end{array}\right)}_2 \leq kL
\end{equation*}
as long as $\norm{C_{ij}}_2 \leq L$ for all $i,j$. Let $W = kL \cdot I - C$. To show that $W$ is positive semi-definite consider 
\begin{equation*}
    W_1:=\left(\begin{array}{ccccc}
    kL \cdot I - C_{11}     & - C_{12} & -C_{13}& \dots & -C_{1k}\\
    -C_{21}     &  L \cdot I  &0 & \dots & 0\\
    \vdots & \vdots & \vdots &\vdots\\ 
    -C_{k1}     &  0 &0 & \dots & L \cdot I 
    \end{array}\right),
    W_2:=W - W_1.
\end{equation*}
By induction, it follows that $W_2$ is positive semi-definite. We need to show that the same holds for $W_1$. By Schur complement we obain that $W_1$ is positive semi-definite if and only if $kL \cdot I - C_{11} - \frac{1}{L}\sum_{i=2}^k C_{1i}C_{i1}$ is positive semi-definite. It follows that \[\norm{ C_{11} - \frac{1}{L}\sum_{i}C_{1i}C_{i1}}_2 \leq \norm{C_{11}}_2 + \frac{1}{L}\sum_{i=2}^k \norm{C_{1i}}_2 \norm{C_{i1}}_2 \leq L + (k-1)L = kL. \]
Hence $W_1$ is positive semi-definite and the induction is complete.
\end{proof}

Returning to the statement of \Cref{claim:smoothness}, we will show that
\begin{equation}\label{eq:needtoshow}
 \norm{\nabla^2_{\pi_j\pi_i} V^j_{\mu}}_2 \leq C,
\end{equation}
 for all $i,j\in \mathcal{N}$ with $C$ chosen to be $\frac{2\gamma A_{\max}}{(1-\gamma)^3}$. Assuming we have shown (\ref{eq:needtoshow}), we conclude from Claim \ref{claim:algebra} that 
\[
\norm{\nabla^2 \Phi_{\mu}}_2 \leq nC,
\]
and hence $\Phi$ will be $nC$-smooth (the proof of Lemma \ref{claim:smoothness} will follow).

To prove (\ref{eq:needtoshow}), we follow the same proof steps as in the proof of \cite{Aga20}, Lemma D.3. We will need to prove an upper bound on the largest eigenvalue (in absolute value) of the matrix 
\[\nabla^2_{\pi_j \pi_i} V^j_{\mu} = \nabla^2_{\pi_j \pi_i} V^i_{\mu},\]
along the direction where only agent $i$ is allowed to change policy. 

Fix policy $\pi= (\pi_1,...,\pi_n)$, agents $i \neq j$, scalars $t,s \geq 0$, state $s_0$ and $u,v$ be unit vectors such that $\pi_i + t\cdot u \in \Delta(\mathcal{A}_i)^S$ and $\pi_j + s\cdot v \in \Delta(\mathcal{A}_j)^S$. Moreover, let $V(t) = V^i_{s_0}(\pi_i + t\cdot u,\pi_{-i}).$ and $W(t,s) = V^i_{s_0}(\pi_i + t\cdot u,\pi_j + s\cdot v,\pi_{-i,-j}).$  It suffices to show that 
\begin{equation}\label{eq:suffices}
\max_{\norm{u}_2=1} \left|\frac{d^2 V(0)}{dt^2} \right| \leq \frac{2\gamma |\mathcal{A}_i|}{(1-\gamma)^3} \textrm{ and }\max_{\norm{u}_2=1} \left|\frac{d^2 W(0,0)}{dt ds} \right|\leq \frac{2\gamma \sqrt{|\mathcal{A}_i||\mathcal{A}_j|}}{(1-\gamma)^3}.
\end{equation}
\begin{itemize}
\item We first focus on $V(t)$. 
\end{itemize}

It holds that $V(t) = \sum_{a \in \mathcal{A}_i}\sum_{\a \in \mathcal{A}_{-i}} (x_{i,s_0,a}+t u_{i,s_0,a})  \prod_{j\neq i} x_{j,s_0,a_j} Q^i_{s_0}((\pi_i+tu,\pi_{-i}),(a,\a))$ (note that $\sum_{a \in \mathcal{A}_i}\sum_{\a \in \mathcal{A}_{-i}} (x_{i,s_0,a}+t u_{i,s_0,a})  \prod_{j\neq i} x_{j,s_0,a_j}=1$ since it is a distribution), hence taking the second derivative we have
\begin{equation}\label{eq:VQ}
\begin{split}
\frac{d^2 V(0)}{dt^2} = &\sum_{a \in \mathcal{A}_i}\sum_{\a \in \mathcal{A}_{-i}} (x_{i,s_0,a}+t u_{i,s_0,a})  \prod_{j\neq i} x_{j,s_0,a_j}\frac{d^2 Q^i_{s_0}(\pi,(a,\a))}{dt^2}
\\&+ 2\sum_{a \in \mathcal{A}_i}\sum_{\a \in \mathcal{A}_{-i}}  u_{i,s_0,a}  \prod_{j\neq i} x_{j,s_0,a_j}\frac{d Q^i_{s_0}(\pi,(a,\a))}{dt}  
\end{split}
\end{equation}
For the remaining of the first part of the proof, we shall show $$\left|\frac{d Q^i_{s_0}(\pi,(a,\a))}{dt}\right| \leq \frac{\gamma \sqrt{|\mathcal{A}_i|}}{(1-\gamma)^2} \textrm{ and }\left|\frac{d^2 Q^i_{s_0}(\pi,(a,\a))}{dt^2}\right| \leq \frac{2\gamma^2 |\mathcal{A}_i|}{(1-\gamma)^3},$$
and then combining with (\ref{eq:VQ}) we get 
\begin{equation*}
\begin{split}
\left|\frac{d^2V(0)}{dt^2}\right| &\leq \frac{2\gamma \sqrt{|\mathcal{A}_i|}}{(1-\gamma)^2} \sum_{a \in \mathcal{A}_i} |u_{i,s_0,a}| + \frac{2\gamma^2 |\mathcal{A}_i|}{(1-\gamma)^3} \\&\leq \frac{2\gamma |\mathcal{A}_i|}{(1-\gamma)^2} +  \frac{2\gamma^2 |\mathcal{A}_i|}{(1-\gamma)^3} \\&\leq \frac{2\gamma A_{\max}}{(1-\gamma)^3}\,. \end{split}
\end{equation*}
To bound the derivative of the $Q$-function, observe that
$Q^i_{s_0}((\pi_i+tu,\pi_{-i}),(a,\a)) = e^{\top}_{s_0,a}(I - \gamma P(t))^{-1}r$, where $r(s_0,a)$ is the expected reward of agent $i$ (w.r.t the randomness of the remaining agents) if he chooses action $a$ at state $s_0$ and $P(t)$ is state-action transition matrix of w.r.t the joint distribution of all agents but $i$, i.e., $\pi_{-i}$ and the environment. 

It is clear that $\frac{d^2 P}{dt^2} =0$ (linear with respect to $t$ because of direct parametrization) and moreover $\norm{\frac{dP}{dt}}_{\infty} \leq \sum_{a \in \mathcal{A}_i} |u_{i,s_0,a}| \leq \sqrt{\mathcal{A}_{i}} \leq \sqrt{A_{\max}}.$ Using the fact that $\norm{(I - \gamma P(t))^{-1}}_{\infty} \leq \frac{1}{1-\gamma},$  
we get 
\begin{equation}\label{eq:firstQ}
\begin{split}
    \left|\frac{d Q^i_{s_0}(\pi,(a,\a))}{dt}\right| &= \gamma \left|e^{\top}_{s_0,a}(I - \gamma P(0))^{-1}\frac{dP(0)}{dt}(I - \gamma P(0))^{-1}r\right|
    \leq \frac{\gamma\sqrt{|\mathcal{A}_{i}|}}{(1-\gamma)^2} \leq \frac{\gamma \sqrt{A_{\max}}}{(1-\gamma)^2},
\end{split}
\end{equation}
and also
\begin{equation}\label{eq:secondQ}
\begin{split}
    \left|\frac{d^2 Q^i_{s_0}(\pi,(a,\a))}{dt^2}\right| &= 2\gamma^2 \left|e^{\top}_{s_0,a}(I - \gamma P(0))^{-1}\frac{dP(0)}{dt}(I - \gamma P(0))^{-1}\frac{dP(0)}{dt}(I - \gamma P(0))^{-1}r\right|
    \\& \leq \frac{2\gamma^2|\mathcal{A}_{i}|}{(1-\gamma)^3} \leq \frac{2\gamma^2 A_{\max}}{(1-\gamma)^3},
\end{split}
\end{equation}
Since $u$ is arbitrary, the first part of (\ref{eq:suffices}) is proved.
\begin{itemize}
\item For the second part, we focus on $W(t)$ which is equal to 
\end{itemize}
\begin{equation}\label{eq:W}
\begin{split}
W(t,s) = &\sum_{a \in \mathcal{A}_i}\sum_{b \in \mathcal{A}_j}\sum_{\a \in \mathcal{A}_{-i,-j}} (x_{i,s_0,a}+t u_{i,s_0,a})(x_{j,s_0,b}+s v_{j,s_0,b}) \cdot  \\&\cdot \prod_{j'\neq i,j} x_{j',s_0,a_{j'}} Q^i_{s_0}((\pi_i+tu,\pi_j+sv,\pi_{-i,-j}),(a,b,\a))
\end{split}
\end{equation}
We consider the derivative of $W$ (\ref{eq:W}) and we get
\begin{equation}
\begin{split}
\frac{dW(0,0)}{dtds} &= \sum_{a \in \mathcal{A}_i}\sum_{b \in \mathcal{A}_j}\sum_{\a \in \mathcal{A}_{-i,-j}}  u_{i,s_0,a} v_{j,s_0,b} \cdot \prod_{j'\neq i,j} x_{j',s_0,a_{j'}} Q^i_{s_0}(\pi,(a,b,\a))
\\&+\sum_{a \in \mathcal{A}_i}\sum_{\a \in \mathcal{A}_{-i}}  u_{i,s_0,a}  \cdot \prod_{j'\neq i} x_{j',s_0,a_{j'}} \frac{dQ^i_{s_0}(\pi,(a,\a))}{dt}
\\&+\sum_{b \in \mathcal{A}_j}\sum_{\a \in \mathcal{A}_{-j}} v_{j,s_0,b} \cdot \prod_{j'\neq j} x_{j',s_0,a_{j'}} \frac{dQ^i_{s_0}(\pi,(b,\a))}{dt}.
\\&+\sum_{\a \in \mathcal{A}}  \prod_{j'} x_{j',s_0,a_{j'}} \frac{d^2Q^i_{s_0}(\pi,\a)}{dtds}.
\end{split}
\end{equation}
The first term of the sum in absolute value is at most $\frac{\sqrt{|\mathcal{A}_i||\mathcal{A}_j|}}{1-\gamma}$ (assuming rewards lie in $[0,1]$.) Moreover using (\ref{eq:firstQ}) the second term of the sum in absolute value is bounded by $\frac{\gamma\sqrt{|\mathcal{A}_i|}\sqrt{|\mathcal{A}_i|}}{(1-\gamma)^2}$ and the third term by $\frac{\gamma\sqrt{|\mathcal{A}_j|}\sqrt{|\mathcal{A}_i|}}{(1-\gamma)^2}.$ To bound the $\frac{d^2Q^i_{s_0}(\pi,\a)}{dtds}$, the same approach works that we used to prove (\ref{eq:secondQ}) with the extra fact that the state-action transition matrix is $P(t,s)$ and moreover the reward $r(s_0,a,b)$ is the expected reward of agent $i$ (w.r.t the randomness of all agents but $i,j$) if $i$ chooses action $a$ and $j$ chooses $b$ at state $s_0.$

Finally for the fourth term we get that
\begin{equation*}
    \begin{split}
\left|\frac{d^2Q^i_{s_0}(\pi,\a)}{dtds}\right| &\leq \gamma^2\left| e^{\top}_{s_0,a}(I - \gamma P(0,0))^{-1}\frac{dP(0,0)}{ds}(I - \gamma P(0,0))^{-1}\frac{dP(0,0)}{dt}(I - \gamma P(0,0))^{-1}r\right|+
\\&+\gamma^2\left| e^{\top}_{s_0,a}(I - \gamma P(0,0))^{-1}\frac{dP(0,0)}{dt}(I - \gamma P(0,0))^{-1}\frac{dP(0,0)}{ds}(I - \gamma P(0,0))^{-1}r\right|+
\\&+\gamma \left|e^{\top}_{s_0,a}(I - \gamma P(0,0))^{-1}\frac{d^2P(0,0)}{dt^2}(I - \gamma P(0,0))^{-1}r\right|
    \\&\leq \frac{\gamma^2\sqrt{|\mathcal{A}_i||\mathcal{A}_j|}}{(1-\gamma)^3}+\frac{\gamma^2\sqrt{|\mathcal{A}_i||\mathcal{A}_j|}}{(1-\gamma)^3}+\frac{\gamma \sqrt{|\mathcal{A}_i||\mathcal{A}_j|}}{(1-\gamma)^2} \leq \frac{2\gamma \sqrt{|\mathcal{A}_i||\mathcal{A}_j|} }{(1-\gamma)^3} \leq \frac{2\gamma A_{\max} }{(1-\gamma)^3}.\qedhere
    \end{split}
\end{equation*}
\end{proof}

\section{Auxiliary Lemmas}
Recall that $P_\X$ denotes the projection onto some set $\X$.
\begin{lemma}[\cite{bubeck}, Lemma 3.6]\label{lem:descent} 
Let $f$ be a $\beta$-smooth function\footnote{Differentiable with $\nabla f$ to be $\beta$-Lipschitz.} with convex domain $\mathcal{X}$. Let $x\in \mathcal{X}$, $x^{+} = P_{\mathcal{X}} (x - \frac{1}{\beta} \nabla f(x))$ and $g_{\mathcal{X}}(x) = \beta(x - x^{+})$. Then the following holds true:
\[f(x^{+})-f(x) \leq -\frac{1}{2\beta}\norm{g_{\mathcal{X}}(x)}_2^2.\]
\end{lemma}

\begin{lemma}[\cite{Aga20}, Proposition B.1]\label{lem:approxsmooth} Let $f(\pi)$ be a $\beta$-smooth function in $\pi \in \Delta(\mathcal{A})^S$. Define the gradient mapping 
\[G(\pi) = \beta\left(P_{\Delta(\mathcal{A})^S}\left(\pi+\frac{1}{\beta}\nabla_{\pi}f(\pi)\right)-\pi\right)\]
and the update rule for the projected gradient is $\pi' = \pi + \frac{1}{\beta} G(\pi)$. If $\norm{G(\pi)}_2 \leq \epsilon$, then 
\[
\max_{\pi+\delta \in \Delta(\mathcal{A})^S, \norm{\delta}_2\leq 1} \delta^{\top}\nabla_{\pi} f(\pi') \leq 2\epsilon.
\]
\end{lemma}
\begin{lemma}\label{lem:approxsmooth22} Let $\Phi_\mu(\pi)$ be the potential function (which is $\beta$-smooth) and assume $\pi \in \Delta(\A)^S$ uses $\alpha$-greedy parametrization. Define the gradient mapping 
\[G(\pi) = \beta\left(P_{\Delta(\mathcal{A})^S}\left(\pi+\frac{1}{\beta}\nabla \tilde{\Phi}_\mu(\pi)\right)-\pi\right)\]
where $\tilde{\Phi}(\pi)$ is the potential function after $\alpha$-greedy parametrization and the update rule for the projected gradient is $\pi' = \pi + \frac{1}{\beta} G(\pi)$. If $\norm{G(\pi)}_2 \leq \epsilon$, then 
\[
\max_{\pi+\delta \in \Delta(\mathcal{A})^S, \norm{\delta}_2\leq 1} \delta^{\top}\nabla \Phi_\mu(\pi') \leq 2\epsilon + \alpha \frac{2n\gamma A_{\max}}{(1-\gamma)^3}\sqrt{n SA_{\max}} \leq 2\epsilon + \alpha \frac{2 S^{1/2}(nA_{\max})^{3/2}}{(1-\gamma)^3}.
\]
\end{lemma}
\begin{proof} 
It is a direct application of Lemma \ref{lem:approxsmooth} and the fact that $\nabla\Phi_{\mu}$ is Lipschitz with parameter $\frac{2n\gamma A_{\max}}{(1-\gamma)^3}$ (this is \cref{claim:smoothness}).
\end{proof}

\begin{lemma}[Unbiased with bounded variance  \cite{Das20}]\label{lem:unbiaseddas20}
It holds that $\hat{\nabla}_{\pi_i}^{(t)}$ is unbiased estimator of $\nabla_{\pi_i} V^i$ for all $i$, that is
\[\mathbb{E}_{\pi^{(t)}}\hat{\nabla}_{\pi_i}^{(t)} = \nabla_{\pi_i} V^i_{\rho}(\pi^{(t)}) \textrm{ for all }i.\]
Moreover for all agents $i$ we get that (this is what the authors actually prove)
\[\mathbb{E}_{\pi^{(t)}} \norm{\hat{\nabla}_{\pi_i}^{(t)}}_2^2\leq  \frac{24A_{\max}^2}{\alpha (1-\gamma)^4}.\]
\end{lemma}

\section{Additional Experiments}\label{app:experiments}

In this part, we provide a systematic analysis of variations of the experimental setting in \Cref{sec:experiments}.

\paragraph{Coordination beyond MPGs}

As mentioned in \Cref{rem:oMPGs}, we know that policy gradient converges also in other cooperative settings which may fail to be exact MPGs. To study this case, we modify our experiment from \Cref{sec:experiments}. The setting remains mostly the same, except now in the distancing state the rewards are reduced by a differing (yet still sufficiently large) amount, $c_k$, for each facility $k = A,B,C,D$.\par

\begin{figure}[!htb]
    \centering
    \includegraphics[width=0.32\textwidth, height=3.86cm, clip=true, trim=16 2 60 25]{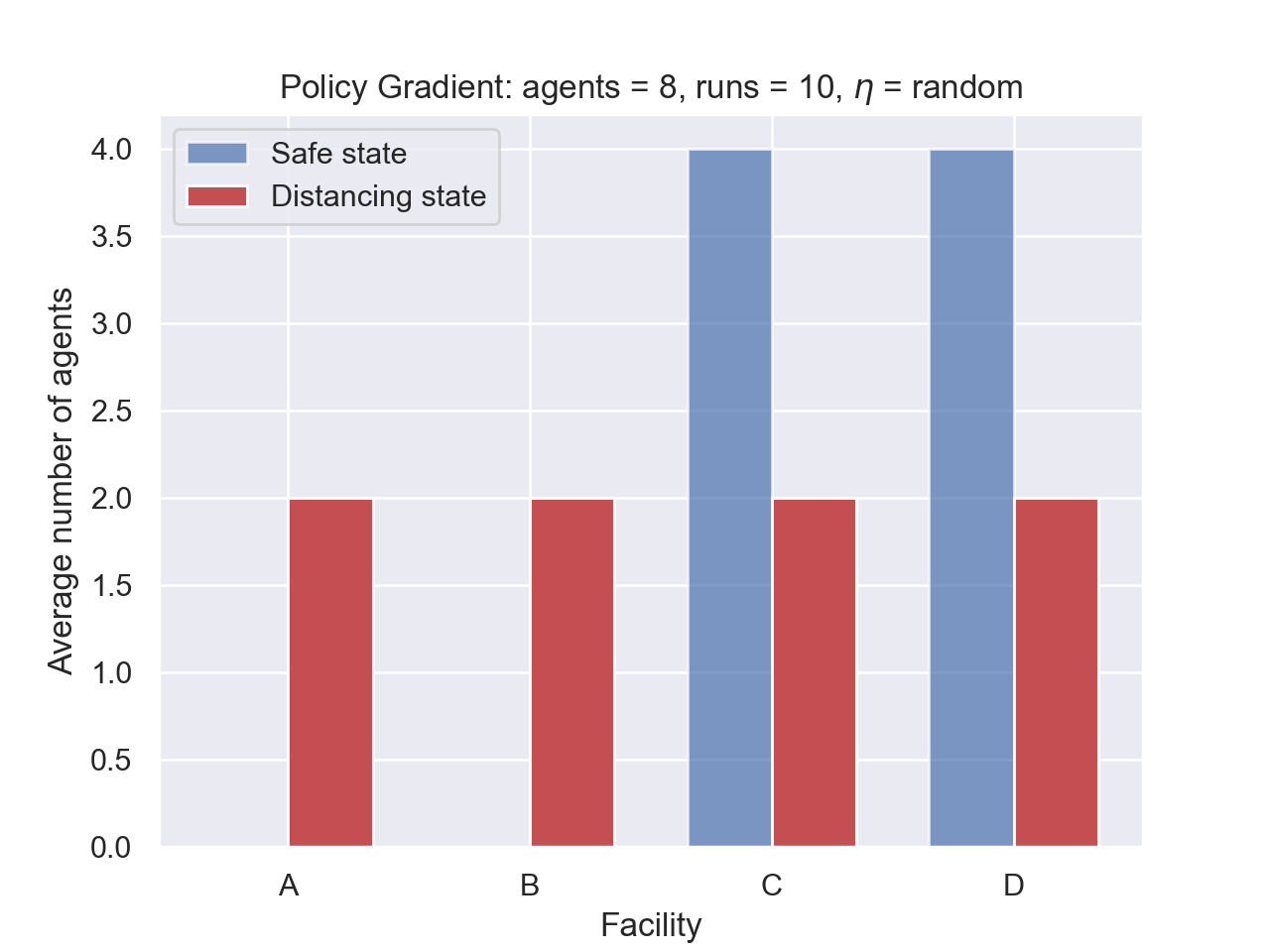}
    \raisebox{0.1em}{
    \includegraphics[width=0.32\textwidth, clip=true, trim=15 0 40 15]{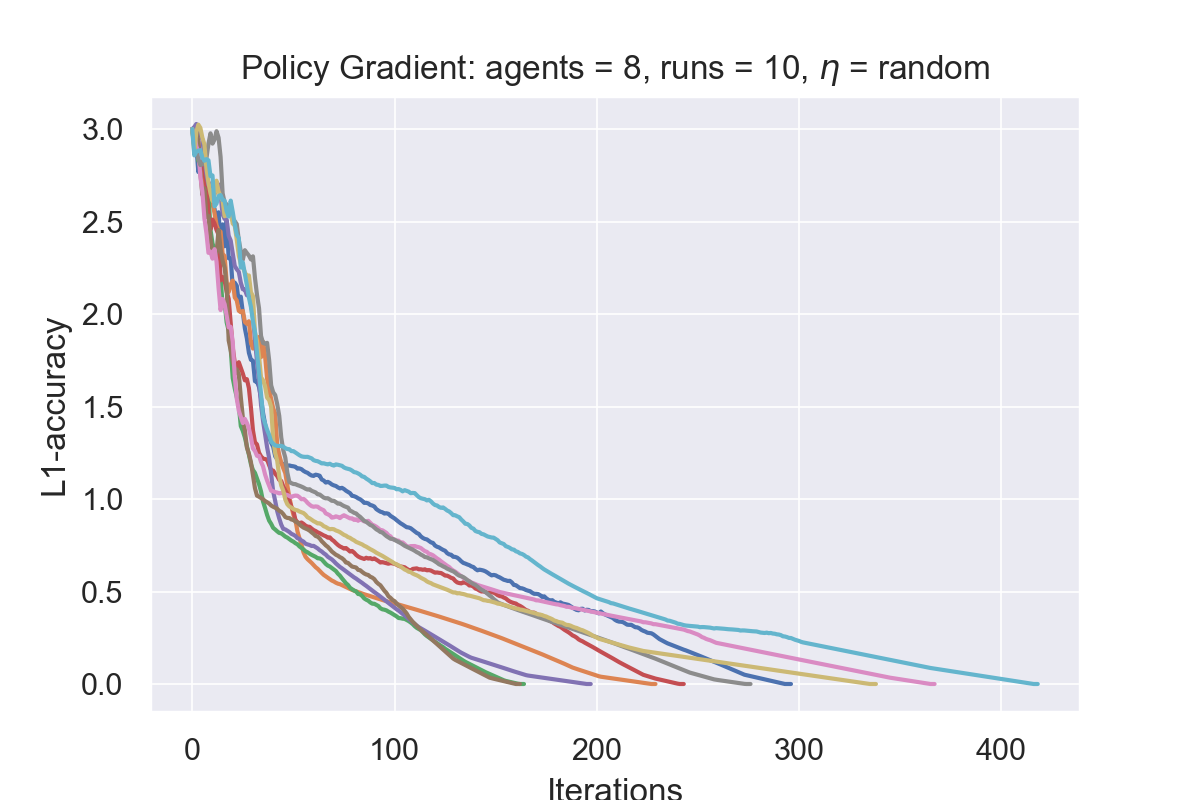}\hspace{0.1cm}
    \includegraphics[width=0.32\textwidth,clip=true, trim=15 0 40 15]{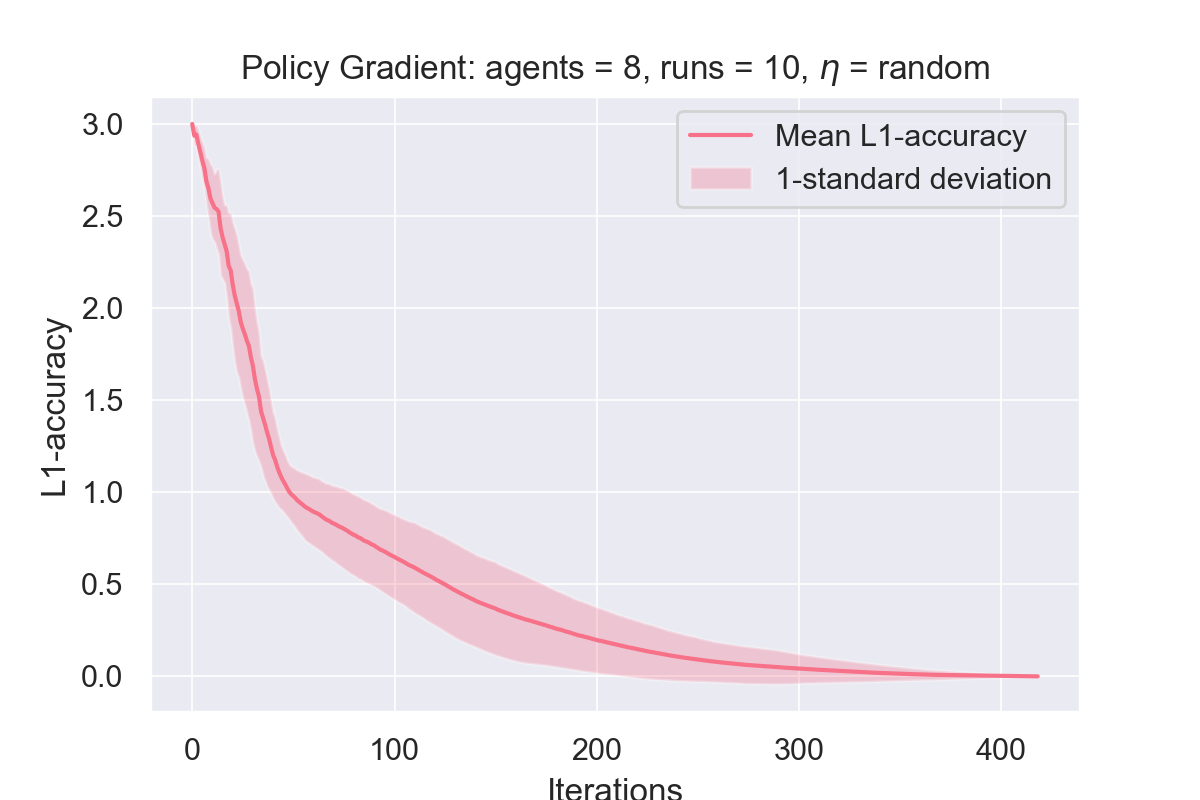}}
    \caption{Figures similar to \Cref{fig:congestion}, except now with $c > c_A > c_B > c_C > c_D$ as described in the text. Independent policy gradient requires more iterations to converge compared to the symmetric shift setting, but still arrives at the same Nash policy.}
    \label{fig:ord_congestion}
\end{figure}

Despite the introduced asymmetry, it is still natural that cooperation is desirable in this setting. In particular, if the $c_k$'s for all $k=A,B,C,D$ are taken to be greater than the $c$ of the MDP from \Cref{sec:experiments}, then the agents can be said to have even \enquote{stronger} incentive to cooperate. The results of running independent policy gradient on this variant are shown in \Cref{fig:ord_congestion}. Independent policy gradient requires more iterations to converge compared to the symmetric setting, but still arrives at the same Nash policy.

\paragraph{Coordination with more agents and facilities}

\begin{figure}[!htb]
    \centering
    \includegraphics[width=0.32\textwidth]{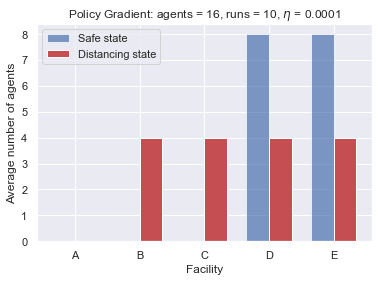}
    \raisebox{0.2em}{ \includegraphics[width=0.32\textwidth]{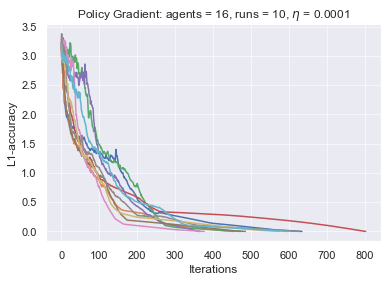}\hspace{0.1cm}
    \includegraphics[width=0.32\textwidth]{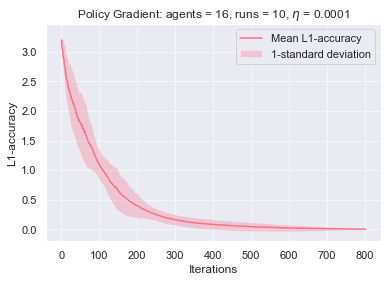}}
    \caption{Convergence to deterministic Nash policies of independent policy gradient in a variation of the MDP of \Cref{sec:experiments} with $N=16$ agents and $A_i=5$ facilities, $\A_i=\{A,B,C,D,E\}$ with $w_A < w_B < w_C < w_D < w_E$ (i.e., $E$ is the most preferable by all agents). Again, while there are several (symmetric) deterministic Nash policies, all of them yield the same distribution of agents among states (leftmost panel). All runs converge successfully to that outcome (however, some runs required a larger number of iterates to converge).}
    \label{fig:more_agents}
\end{figure}

We next test the performance of the independent policy gradient algorithm in a larger setting with $N=16$ agents and $A_i=5$ facilities, $\A_i=\{A,B,C,D,E\}$ with $w_A < w_B < w_C < w_D < w_E$ (i.e., $E$ is the most preferable by all agents). We use a learning rate $\eta=0.0001$ for all agents (which is again much larger than the theoretical guarantee of \Cref{thm:mainformal}). All runs lead to convergence to an (optimal) Nash policy as shown in the middle and rightmost panels. The leftmost panel shows the distribution of the agents among facilities in both states, which is the same (and the optimal one) in all Nash policies that are reached by the algorithm. The results are shown in \Cref{fig:more_agents}.

\paragraph{Coordination with random transitions}

\begin{figure}[!htb]
    \centering
    \includegraphics[width=0.32\textwidth]{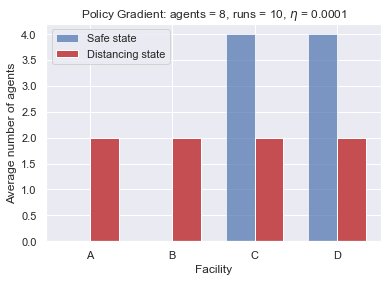}
    \raisebox{0.2em}{ \includegraphics[width=0.32\textwidth]{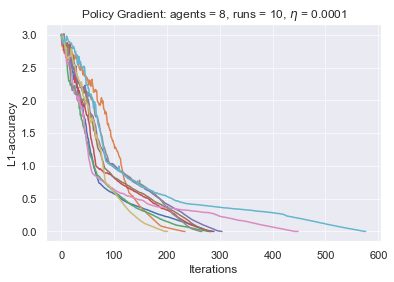}\hspace{0.1cm}
    \includegraphics[width=0.32\textwidth]{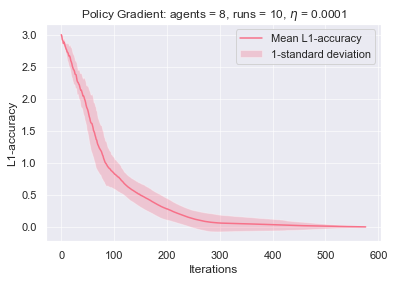}}\\[0.2cm]
    \includegraphics[width=0.32\textwidth]{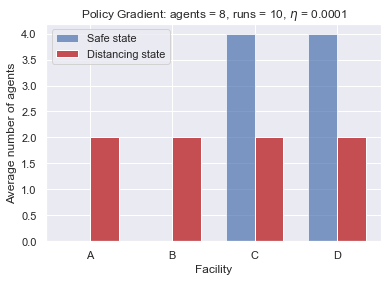}
    \raisebox{0.2em}{ \includegraphics[width=0.32\textwidth]{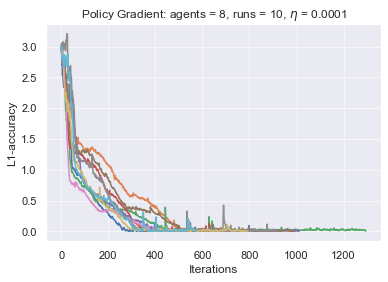}\hspace{0.1cm}
    \includegraphics[width=0.32\textwidth]{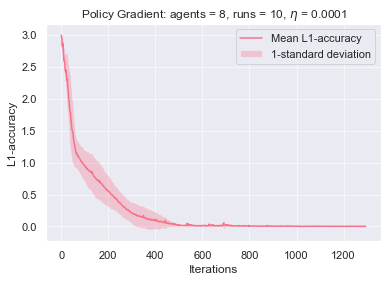}}
    \caption{Convergence to deterministic Nash policies of independent policy gradient in two variations of the MDP of \Cref{sec:experiments} with stochastic transitions between states.}
    \label{fig:ran_tran}
\end{figure}

Next, we study the effect of adding randomness to the transitions on the performance of the individual policy gradient algorithm. In this case, we experiment with the same setting as in \Cref{sec:experiments} (i.e., $N=8$ agents and $A_i=4$ facilities that each agent $i\in \N$ can choose from), but use the following stochastic transition rule instead: in addition to the existing transition rules, the sequence of play may transition from the safe to the distancing state with probability $p\%$ regardless of the distribution of the agents and may remain at the distancing state with probability $q\%$ again regardless of the distribution of the agents there. \par
Two sets of results are presented in \Cref{fig:ran_tran}. In the first (upper panels), we use $p,q=1\%,10\%$ and in the second $p,q=5\%,20\%$. In both cases, we use a learning rate $\eta=0.0001$ (several orders of magnitude higher than what is required by \Cref{thm:mainformalb}). Independent policy gradient converges in both cases to deterministic Nash policies despite the randomness in the transitions. However, for higher levels of randomness (lower panels), the algorithm remains at an $\epsilon$-Nash policy for a high number of iterations. This is in line with the theoretical predictions of \Cref{thm:mainformalb}. 